\documentclass[11pt,a4paper]{article}
\usepackage{amsfonts}
\usepackage{amsthm}
\usepackage{amsmath}
\usepackage{amssymb}
\usepackage{mathrsfs}
\usepackage{dsfont}   
\usepackage{latexsym,epsf,graphicx}
\usepackage{caption}
\usepackage[colorlinks=false]{hyperref}
  \hypersetup{
    pdftitle={Total stability of kernel methods},
    pdfauthor={\textcopyright DRAFT VERSION.  Andreas Christmann, Daohong Xiang, and Ding-Xuan Zhou}
   }
\usepackage[dvips]{color}
\usepackage{url}

\usepackage[sort&compress]{natbib}  
 \bibpunct{(}{)}{;}{a}{,}{,}

\newtheorem{theorem}{Theorem}[section]
\newtheorem{remark}[theorem]{Remark}

\newtheorem{corollary}[theorem]{Corollary}
\newtheorem{definition}[theorem]{Definition}
\newtheorem{lemma}[theorem]{Lemma}

\newtheorem{assumption}[theorem]{Assumption}

\newcommand{\faaa}{f_{\P_1,\lb_1,k_1}}

\newcommand{\fbaa}{f_{\P_2,\lb_1,k_1}}

\newcommand{\fbba}{f_{\P_2,\lb_2,k_1}}
\newcommand{\fbbb}{f_{\P_2,\lb_2,k_2}}
\newcommand{\falk}{f_{\P_1,\lb,k}}
\newcommand{\fblk}{f_{\P_2,\lb,k}}
\newcommand{\flb}{f_{\lb}}
\newcommand{\fmu}{f_{\mu}}

\newcommand{\fsaaa}{f_{\P_1,\lb_1,k_1}}

\newcommand{\fsbaa}{f_{\P_2,\lb_1,k_1}}

\newcommand{\fsbba}{f_{\P_2,\lb_2,k_1}}
\newcommand{\fsbbb}{f_{\P_2,\lb_2,k_2}}
\newcommand{\hanorm}[1] {\Vert #1 \Vert_{H_1}}
\newcommand{\hhanorm}[1] {\Vert #1 \Vert_{H_1}^2}
\newcommand{\hbnorm}[1] {\Vert #1 \Vert_{H_2}}

\newcommand{\kk}[1] {k(\cdot,#1)}
\newcommand{\ka}[1] {k_1(\cdot,#1)}
\newcommand{\kb}[1] {k_2(\cdot,#1)}

\newcommand{\bc}{\begin{center}}
\newcommand{\ec}{\end{center}}
\newcommand{\bi}{\begin{itemize}}
\newcommand{\ei}{\end{itemize}}
\newcommand{\be}{\begin{equation}}
\newcommand{\ee}{\end{equation}}
\newcommand{\beqna}{\begin{eqnarray*}}
\newcommand{\eeqnal}{\end{eqnarray}}
\newcommand{\beqnal}{\begin{eqnarray}}
\newcommand{\eeqna}{\end{eqnarray*}}
\newcommand{\bd}{\begin{displaymath}}
\newcommand{\ed}{\end{displaymath}}
\newcommand{\bt}{\begin{tabular}}
\newcommand{\et}{\end{tabular}}

\newcommand{\myem}[1]{\textbf{#1}}

\newcommand{\rem}[1]{}
\newlength{\fixboxwidth}
\newcommand{\N}{\mathds{N}}

\newcommand{\R}{\mathds{R}}
\newcommand{\Rd}{\R^{d}}

\newcommand{\snorm}[1] {\Vert #1 \Vert}

\newcommand{\hnorm}[1] {\Vert #1 \Vert_{\sH}}
\newcommand{\hhnorm}[1] {\Vert #1 \Vert_{\sH}^2}
\newcommand{\inorm}[1] {\Vert #1 \Vert_\infty}
\newcommand{\tvnorm}[1]{\Vert  #1  \Vert_{tv}} 

\newcommand{\sH}    {H}  
\newcommand{\sA} {\mathcal{A}}     
\newcommand{\B}    {\mathcal{B}}   

\def \lb        { \lambda }
\def \a         { \alpha }

\def \g         { \gamma }

\def \s         { \sigma }

\def \t         { \tau }
\def \d         { \delta }

\def \x         { \xi }

\newcommand{\Lsrisk}[1]{{\cal R}_{\Ls,{\P}}(#1)}       

\newcommand{\Lsriska}[1]{{\cal R}_{\Ls,{\P_1}}(#1)} 
\newcommand{\Lsriskb}[1]{{\cal R}_{\Ls,{\P_2}}(#1)} 

\def \P           { \mathrm{P} }   
\def \Q           { \mathrm{Q} } 

\def \D           { \mathrm{D} }

\newcommand{\PM}  {\mathcal{M}_1}   
\newcommand{\PMXY}  {\mathcal{M}_1(\cX\times \cY)}   

\newcommand{\Ex}{\mathbb{E}}       

\newcommand{\fP}{f_{L,\P,\lb}}

\newcommand{\fLsPk}{f_{\P,\lb,k}}

\def\zeitende{\hfill \quad \hbox{$\vartriangleleft$}}
\def\exampleende{\ifmmode\zeitende\else{\unskip\nobreak\hfil
\penalty50\hskip1em\null\nobreak\hfil\zeitende
\parfillskip=0pt\finalhyphendemerits=0\endgraf}\fi}

\newcommand{\ca}[1]{{\cal #1}}

\DeclareMathOperator{\diam}{diam}

\newcommand{\RPreg}[2]{{{\cal R}_{#1,\P,\lb}^{reg}(#2)}}

\newcommand{\RP}[2]{{{\cal R}_{#1,\P}(#2)}}

\newcommand{\RT}[2]{{{\cal R}_{#1,\D}(#2)}}

\newcommand{\RPLM}{RPL method }

\newcommand{\DfiveL}[1] {D_5 L\bigl(X,Y,\tiX,\tiY, {#1}(X), {#1}(\tiX)\bigr)}
\newcommand{\DsixL}[1] {D_6 L\bigl(X,Y,\tiX,\tiY, {#1}(X), {#1}(\tiX)\bigr)}

\newcommand{\Dfive}[1] {D_5 \Ls \circ {#1} (x,y,\tix,\tiy)}
\newcommand{\Dsix}[1] {D_6 \Ls \circ {#1} (x,y,\tix,\tiy)}

\newcommand{\Ls}{L^{\star}}
\newcommand{\fPLs}{f_{\Ls,\P,\lb}}
\newcommand{\fQLs}{f_{\Ls,\Q,\lb}}

\newcommand{\cH}{H}

\newcommand{\cR}{\mathcal{R}}
\newcommand{\cX}{\mathcal{X}}
\newcommand{\cY}{\mathcal{Y}}

\newcommand{\cXY}{{\cX\times\cY}}

\newcommand{\tix} {\tilde{x}}
\newcommand{\tiy} {\tilde{y}}

\newcommand{\tiX} {\tilde{X}}
\newcommand{\tiY} {\tilde{Y}}

\newcommand{\tit} {\tilde{t}}
\newcommand{\xyxy} {x,y,\tix,\tiy}

\newenvironment{declaration}[1]{\trivlist \item[\hskip \labelsep{\em #1 }]\ignorespaces}{\endtrivlist}
\newenvironment{proofof}[1]{\begin{declaration}{#1.}}{\end{declaration}}

\def \O { \Omega }

\newcommand{\bnum}{\begin{enumerate}}
\newcommand{\enum}{\end{enumerate}}

\def \lb        { \lambda }

\newcommand{\beq}{\begin{eqnarray}}
\newcommand{\eeq}{\end{eqnarray}}

\newcommand{\XYXY}{X,Y,\tiX,\tiY}

\newcommand{\qedr}{\hfill \quad \qed}

\newcommand{\dPro}{{{d}_{\mathrm{Pro}}}} 
\def\textregtrademark{\raise1ex\hbox{\scriptsize\textregistered}}

\newcommand{\E}{{\Ex}}

\numberwithin{equation}{section}

\begin{document}

\title{\textbf{Total stability of kernel methods}$^\dag$\footnotetext{\dag
Corresponding author: Daohong Xiang, Email: \url{daohongxiang@zjnu.cn}\newline
~The work by A. Christmann described in this paper is partially supported by two grants of the Deutsche
Forschungsgesellschaft [Project No. CH/291/2-1 and CH/291/3-1].
The work by D. H. Xiang described in this paper is supported by the National Natural Science Foundation of China under Grant 11471292 and the Alexander von Humboldt Foundation of Germany.
The work by D.-X. Zhou described in this paper is supported partially by  the Research Grants Council of Hong Kong under project $\#$ CityU 11304114.}}

\author{\textbf{Andreas Christmann}$^1$,
\textbf{Daohong Xiang}$^2$,
and \textbf{Ding-Xuan Zhou}$^3$\\
$^1$ University of Bayreuth, Germany,\\
$^2$ Zhejiang Normal University, China, and University of Bayreuth, Germany,\\
$^3$ City University of Hong Kong, China}
\date{Date: \today}

\maketitle

\begin{abstract} \noindent
Regularized empirical risk minimization using kernels and their corresponding reproducing kernel Hilbert spaces (RKHSs) plays an important role in machine learning. However, the actually used kernel often depends on one or on a few hyperparameters or
the kernel is even data dependent in a much more complicated manner. Examples are Gaussian RBF kernels, kernel learning, and hierarchical Gaussian kernels which were recently proposed for deep learning. Therefore, the actually used kernel is often
computed by a grid search or in an iterative manner and can often only be considered as an approximation
to the ``ideal'' or ``optimal'' kernel. \newline
The paper gives conditions under which classical kernel based methods based on a convex Lipschitz loss function and on a bounded and smooth kernel are stable, if the probability measure $\P$, the regularization parameter $\lambda$, and
the kernel $k$ may slightly change in a \emph{simultaneous} manner.
Similar results are also given for pairwise learning.
Therefore, the topic of this paper is somewhat more general than in classical robust statistics, where usually only the influence of small perturbations of the probability measure $\P$ on the estimated function is considered.
\end{abstract}

\noindent{\bf Key words and phrases.} Machine learning; stability; robustness; kernel; regularization.

\noindent{\bf AMS Subject Classification Numbers.} 68Q32, 62G35, 68T05, 68T10, 62M20.


\section{Introduction}\label{intro}
Regularized empirical risk minimization using the kernel approach including support vector machines (SVMs) based on a general convex loss function and regularized pairwise learning (RPL) methods plays a very important role in machine learning. Such kernel methods have been widely investigated from the points of view of universal consistency, learning rates, and statistical robustness,
see e.g.
\citet{Vapnik1995,Vapnik1998},
\citet{SchoelkopfSmola2002},
\citet{CuckerSmale2002},
\citet{CuckerZhou2007},
\citet{SC2008}, and the references cited in these books.
In short words, universal consistency describes the property that the statistical method or the algorithm converges to the asymptotical optimal value of interest (i.e. the Bayes risk or the Bayes decision function) for \emph{all} probability measures $\P$, if the sample size $n$ converges to infinity and if the regularization parameter $\lambda_{n}$ converges in an appropriate manner to $0$, if $n\to\infty$.
Unfortunately, it turns out by the so-called no-free-lunch theorem
shown by \citet{Devroye1982} that universally consistent methods can in general not have a \emph{uniform rate} of convergence for \emph{all} $\P$.
However, there is a vast literature that regularized empirical risk minimization based on kernels yields optimal guaranteed rates of convergence on \emph{large subsets} of the
set $\PM$ of all probability measures, see e.g.
\citet{CuckerSmale2002},
\citet{SmaleZhou2007},
\citet{CaponnettoDevito2007},
\citet{XiangZhou2009},
\citet{SteinwartHushScovel2009},
and the references cited therein.
Results on the statistical robustness or on various notations of stability have shown that under weak conditions on the loss function $L$ and on the kernel $k$ or its RKHS $H$, many regularized
empirical risk minimization methods including general SVMs and RPL methods are stable with respect to small changes in the probability measure $\P$ or w.r.t. small changes of the data set, see e.g.
\citet{BousquetElisseeff2001},
\citet{ChristmannSteinwart2004a, ChristmannSteinwart2007a},
\citet{PoggioRifkinMukherjeeNiyogi2004},
\citet{MukherjeeNiyogiPoggioRifkin2006},
\citet{ChristmannSalibianBarreraVanAelst2013},
\citet{HableChristmann2011},
\citet{Hable2012}, \citet{ChristmannZhou2016a} and the references cited therein.
Such kernel methods can often be represented by
operators which are continuous or differentiable (in the sense of G{\^a}teaux or Hadamard) \emph{all} probability measures $\P$.

The aim of the present paper is to take a step further: we establish some total stability
results which show that many regularized empirical risk minimization methods based on kernels are even stable, if the full triple $(\P,\lb,k)$ consisting of the -- of course completely unknown -- underlying probability measure $\P$, the regularization parameter $\lambda$, and the kernel $k$ (or its RKHS $H$) changes slightly.
Our main results are Theorem \ref{sec1.thm2}, Corollary \ref{sec1.cor1}, and Theorem \ref{sec1.thm3} for classical loss functions and  Theorem \ref{sec2.thm1}, Corollary \ref{sec2.cor1}, and Theorem \ref{sec2.thm2}
for pairwise learning.
In particular, we establish results like 
\be\label{rr}
\inorm{f_{\P_{1},\lb_{1},k_{1}}-f_{\P_{2},\lb_{2},k_{2}}}
=
    \mathcal{O}\bigl(\tvnorm{\P_1-\P_2}\bigr)
   + \mathcal{O}\bigl(|\lb_1 -\lb_2|\bigr)
   + \mathcal{O}\bigl(\sup_{x\in\cX} (\inorm{k_2(\cdot,x) - k_1(\cdot,x)}) \bigr),
\ee
where $f_{\P_{j},\lb_{j},k_{j}}$ denotes the regularized empirical risk minimization method for the triple $(\P_{j},\lb_{j},k_{j})$, $j\in\{1,2\}$, and
$\tvnorm{\P_{1}-\P_{2}}$ denotes the norm of total variation between the two probability measures. We explicitly give the constants in (\ref{rr}), although the constants may not be optimal.

The rest of the paper has the following structure.
Section 2 yields results for general SVM-type methods based on a classical loss function
$L(x,y,f(x))$.
Section 3 yields similar results for pairwise learning based on functions of the form $L(x,\tix,y,\tiy,f(x),f(\tix))$. Section 4 gives some examples of practical importance. Gaussian RBF kernels and the recently introduced hierarchical Gaussian
RBF kernels for deep learning, see \citet{SteinwartThomannSchmid2016}, are covered by our results.
Section 5 contains a short discussion.
All proofs are given in the appendix.
As this is a theoretical paper, we omit numerical examples.

\section{Results for SVMs}\label{sec1}
In this section we show that many kernel based methods like SVMs have nice total stability properties if simultaneously the distribution $\P$, the regularization parameter $\lb$ and the kernel $k$ slightly change.

\begin{assumption}\label{sec2.assumption-spaces1}
Let $\cX$ be a complete separable metric space and
$\cY\subset\R$ be closed.
Let $(X,Y)$ and $(X_i,Y_i)$, $i\in\N$, be independent and identically
distributed pairs of random quantities with values in $\cXY$.
We denote the joint distribution of $(X_i,Y_i)$ by $\P\in\PM(\cXY)$, where $\PMXY$ is the set of all Borel probability measures on the Borel $\s$-algebra $\B_\cXY$.
\end{assumption}

Let $k: \cX\times \cX\to \mathds{R}$ be a continuous, symmetric and positive semidefinite function, i.e., for any finite set of distinct points $\{x_1, \ldots, x_n\}\subset \cX,$ the kernel matrix $(k(x_i, x_j))_{i,j=1}^{n}$ is positive semidefinite. Such a function is called a \emph{Mercel kernel}. The \emph{reproducing kernel Hilbert space (RKHS)} $\cH$ associated with the kernel $k$ is defined in \cite{Aronszajn1950} to be the completion of the linear span of the set of functions $\{k(\cdot, x): x\in \cX\}$ with the inner product $\langle \cdot, \cdot\rangle_\cH$ given by $\langle \Phi(x),  \Phi(y)\rangle_\cH=k(x,y),$ where $\Phi(x):=k(\cdot, x)$ denotes the canonical feature map of $k$, $x\in \cX.$ RKHSs are interesting, because they satisfy the reproducing property
\be\label{sec1.reproducingproperty}
 \langle \Phi(x), f\rangle_\cH=f(x), \quad x\in \cX, f\in \cH.
\ee

\begin{assumption}\label{sec2.assumption-kernel1}
Let $k, k_1, k_2 :\cX\times \cX \to \R$ be continuous and bounded kernels
with reproducing kernel Hilbert space $\cH, H_1, H_2$, respectively.
Define
$\inorm{k}:=\sup_{x\in \cX}\sqrt{k(x,x)}\in(0,\infty), \inorm{k_j}:=\sup_{x\in \cX}\sqrt{k_j(x,x)}\in(0,\infty)$ for $j\in\{1, 2\},$
and denote $\kappa=\max\{\inorm{k_1}, \inorm{k_2}\}.$ Denote the corresponding canonical feature maps by
$\Phi_j(x), j\in\{1, 2\}.$
\end{assumption}

%


A function $L: \cX\times\cY\times\R\to [0, \infty)$ is called a loss function if $L$ is measurable. Because constant loss functions are not useful for applications, we will always assume that $L$ is not a constant function.

A loss function $L(x, y, t)$ is usually represented by a \emph{margin-based} loss function $\tilde{L}(yt)$ for classification and represented by a \emph{distance-based} loss function $\tilde{L}(y-t)$ for regression if $\tilde{L}: \R\to [0, \infty)$ is a measurable function. For example,
the hinge loss $L_{\text{hinge}}(x,y,t)=\max\{0, 1-yt\}$ and the logistic loss $L_{\text{c-logist}}(x, y, t)=\ln(1+\exp(-yt))$ for classification, the $\epsilon$-insensitive loss $L_{\epsilon\text{-insens}}(x, y, t)=\max\{0, |y-t|-\epsilon\}$ for some $\epsilon>0,$ the Huber's loss $L_{\a\text{-Huber}}(x, y, t)=\begin{cases}
0.5(y-t)^2  & \text{if  }|y-t|\leq \a\\
\a|y-t|-0.5\a^2 &  \text{if  } |y-t|> \a
\end{cases}$ for some $\a>0$ and the logistic loss $L_{\text{r-logist}}(x, y, t)=-\ln\frac{4\exp(y-t)}{(1+\exp(y-t))^2}$ for regression, the pinball loss $L_{\tau\text{-pin}}(x, y, t)=\begin{cases}
(\t-1)(y-t)  & \text{if  }|y-t|< 0\\
\t(y-t) &  \text{if  } |y-t|\geq 0
\end{cases}$ for some $\t>0$ for quantile regression. We refer to \citet{Vapnik1995,  Vapnik1998}, \citet{SchoelkopfSmola2002}, \citet{BerlinetThomasAgnan2004}, \citet{CuckerZhou2007}, \citet{ SC2008}, \citet{ShiFengZhou2011}, and \citet{ZuoLiChen2015}
 for details and more examples of kernels.

\begin{definition}\label{sec1.def1}
The loss function $L$ is called Lipschitz continuous, if there exists a constant $|L|_1<\infty$ such that
\be\label{sec1.Lipschitz}
|L(x, y, t_1)-L(x, y, t_2)|\leq |L|_1|t_1-t_2| \quad \forall x\in \cX, y\in\cY, t_1, t_2\in\R.
\ee
\end{definition}
\begin{assumption}\label{sec1.assumption2}
Let $L$ be a convex with respect to the last argument and Lipschitz continuous loss function with Lipschitz constant $|L|_1\in (0, \infty).$
\end{assumption}

\begin{assumption}\label{sec1.loss}
For all $(x, y)\in\cXY,$ let $L(x, y, \cdot)$ be differentiable and its derivative be Lip\-schitz continuous with Lipschitz constant $|L^\prime|_1\in (0, \infty).$
\end{assumption}

The moment condition $\Ex_{\P}L(X, Y, 0)<\infty$ excludes heavy-tailed distributions such as  the Cauchy distribution and many other stable distributions used in financial or actuarial
problems. We avoid the moment condition by shifting the loss with by the term  $L(x, y, 0).$ This trick is well-known in the literature on robust
statistics, see, e.g., \citet{Huber1967}, \citet{ChristmannVanMessemSteinwart2009}, and \citet{ChristmannZhou2016a}.

Denote the shifted loss function of $L$ by
$$\Ls(x, y, t):=L(x, y, t)-L(x, y, 0),\,  (x, y, t)\in \cX\times\cY\times\R.$$
The shifted loss  function $\Ls$ still shares the properties of $L$ specified in Assumption  \ref{sec1.assumption2} and Assumption \ref{sec1.loss}, see \citet[Proposition 2]{ChristmannVanMessemSteinwart2009},  in particular,  if $L$ is convex, differentiable, and Lipschitz continuous with Lipschitz constant $|L|_1$ with respect to the third argument, then $\Ls$ inherits convexity, differentiability and Lipschitz continuity from $L$ with identical Lipschitz constant $|\Ls|_1=|L|_1.$ Additionally, if the derivative $L^\prime$ satisfies Lipschitz continuity with Lipschitz constant $|L^\prime|_1$, so does $(\Ls)^\prime$ with the identical Lipschitz constant  $|(\Ls)^{\prime}|_1=|L^\prime|_1.$

The SVM associated with $\Ls$ can be defined to solve a minimization problem as follows
\be\label{sec1.shiftedsvm}
f_{\P, \lb, k}:=\arg\min_{f\in H}\big(\Ex_{\P}\Ls(X, Y, f(X))+\lb\hhnorm{f}\big),
\ee
where $\P \in \PM(\cXY),$ $\cH$ is the RKHS of a kernel $k,$ and $\lb>0$ is a regularization parameter to avoid overfitting.

Although the shifted loss function $\Ls$ changes the objective function of SVMs, the minimizers defined by $\Ls$ and $L$ respectively are the same for all $\P\in\mathcal{M}_1(\cX\times \cY)$ and in particular for all empirical distributions $\D$ based on a data set consisting
of $n$ data points $(x_i,y_i)$, $1 \le i \le n,$ if the minimizer of an SVM in terms of $L$ instead of $\Ls$ exists.

Our first main result states that the kernel based estimator $f_{\P, \lb, k}$ defined by (\ref{sec1.shiftedsvm}) only changes slightly if the regularization parameter wiggles a little bit. \citet[Theorem 1]{YeZhou2007} proved the assertion of the following result for margin-based loss functions for classification. Here we show it holds true for more general loss functions.

\begin{theorem}\label{sec1.thm1}
Let Assumptions \ref{sec2.assumption-spaces1}, \ref{sec2.assumption-kernel1}, \ref{sec1.assumption2} and \ref{sec1.loss} be satisfied. Let $f_{\P, \lb, k}$ and $f_{\P, \mu, k}$ be defined by (\ref{sec1.shiftedsvm}). For all $\lb>0, \mu>0,$ we have

$$\hnorm{f_{\P, \lb, k}-f_{\P, \mu, k}}\leq \frac{1}{2}\Big(\frac{\max\{\lb, \mu\}}{\min\{\lb, \mu\}}-1\Big)\big(\hnorm{f_{\P, \lb, k}}+\hnorm{f_{\P, \mu, k}}\big).$$
If there exists a constant $r\in (0, \infty)$ such that $\min\{\lb, \mu\}>r,$ then
$$\hnorm{f_{\P, \lb, k}-f_{\P, \mu, k}}\leq \frac{|L|_1\inorm{k}}{r^2}\cdot |\lb-\mu|= \mathcal{O}( |\lb-\mu| ) \,.$$
\end{theorem}

In order to present our total stability theorem for kernel based methods like SVMs, we first recall the definition for the norm of total variation of two probability measures $\P, \Q\in \PM(\cXY):$
$$d_{tv}(\P, \Q):=\sup_{A\in\B_\cXY}|\P(A)-\Q(A)|=\frac{1}{2}\sup_{h}\Big|\int h\ d\P-\int h\ d\Q\Big|,$$
where the supremum is with respect to all measurable functions $h: \cX\times\cY\in\R$ with $\inorm{h}\leq 1.$

The following total stability theorem states that if both regularization parameters are greater than some specified constant, the supremum norm of the difference $\fsaaa-\fsbbb$ varies in a smooth manner, if $(\P, \lb, k)$ changes only slightly.

\begin{theorem}\label{sec1.thm2}
Let Assumptions \ref{sec2.assumption-spaces1}, \ref{sec2.assumption-kernel1}, \ref{sec1.assumption2}, and \ref{sec1.loss} be satisfied. If $\min\{\lb_1, \lb_2\}>r:=\frac{1}{2}\kappa^2|L^\prime|_1,$ then
\be\label{sec1.supnorm}
\inorm{\fsaaa-\fsbbb}\leq c_1(L)\cdot\tvnorm{\P_1-\P_2}+c_2(L)\cdot |\lb_1-\lb_2|+ c_3(L, \lb_1, \lb_2)\cdot\sup_{ x\in\cX}\inorm{\kb{x}-\ka{x}},
\ee
where $c_1(L):=\frac{2|L|_1}{|L^\prime|_1},\ c_2(L):=\frac{4|L|_1}{\kappa^2|L^\prime|_1^2},$ and $c_3(L, \lb_1, \lb_2):=\frac{|L|_1}{2(\min\{\lb_1, \lb_2\}-r)}\,.$
\end{theorem}

\begin{remark}
\bnum
\item Many popular loss functions satisfy Assumptions \ref{sec1.assumption2} and \ref{sec1.loss}. Three important examples are the logistic loss $L_{\text{c-logist}}$ for classification, the Huber's loss $L_{\a\text{-Huber}}$ and the logistic loss $L_{\text{r-logist}}$ for regression. These three loss functions as well as their first order derivatives with respect to the last argument are Lipschitz continuous with Lipschitz constants $|L_{\text{c-logist}}|_1=1, |L_{\text{c-logist}}^\prime|_1=\frac{1}{4}, |L_{\a\text{-Huber}}|_1=\a, |L_{\a\text{-Huber}}^\prime|_1=1, |L_{\text{r-logist}}|_1=1,$ and $|L_{\text{r-logist}}^\prime|_1=\frac{1}{2},$ respectively.

\item Unfortunately, we cannot prove (\ref{sec1.supnorm}) holds true for all $\lb>0.$ Note that the RKHS-norm of the difference $\fsaaa-\fsbbb$ is undefined here because Theorem \ref{sec1.thm2} does not assume any relationship between $H_1$ and $H_2.$ We will later consider the case that $H_2\subseteq H_1.$
\enum
\end{remark}

Denote the $\Ls$-risk of $f$ by $\Lsrisk{f}=\Ex_{\P}\Ls(X, Y, f(X)).$ In the following corollary we establish total stability also in terms of the $\Ls$-risk.

\begin{corollary}\label{sec1.cor1}
	Let the assumptions of Theorem \ref{sec1.thm2} be satisfied and the constants $c_1(L), c_2(L)$ and $c_3(L, \lb_1, \lb_2)$ be defined in the same manner. Define $r:=\frac{1}{2}\kappa^2|L^\prime|_1.$ If there
	exists a constant $s$ such that
	$0 < s < \min\{\lb_1,\lb_2\} - r$,
	then
\be\label{sec1.supnorm-func}
\inorm{\fsaaa-\fsbbb}\leq c_1(L)\cdot\tvnorm{\P_1-\P_2}+c_2(L)\cdot |\lb_1-\lb_2|+ \bar{c}_3(L)\cdot\sup_{ x\in\cX}\inorm{\kb{x}-\ka{x}},
\ee
and
\beqnal \label{sec1.supnorm-risk}
& &\big|\Lsriska{\fsaaa}-\Lsriskb{\fsbbb}\big|\\
&\leq & c_4(L)\cdot\tvnorm{\P_1-\P_2}+c_5(L)\cdot |\lb_1-\lb_2|+ c_6(L)\cdot\sup_{ x\in\cX}\inorm{\kb{x}-\ka{x}},\nonumber
\eeqnal
where
$\bar{c}_3(L):=\ \frac{|L|_1}{2s},\ c_4(L):=\frac{4|L|_1^2}{|L^\prime|_1},\ c_5(L):=\frac{4|L|_1^2}{\kappa^2|L^\prime|_1^2},$ and $c_6(L):=\frac{|L|_1^2}{2s}\,.$

Therefore, both terms $\inorm{\faaa - \fbbb}$ and $\big|\Lsriska{\fsaaa}-\Lsriskb{\fsbbb}\big|$ are of the order
	\be \label{sec1.bigo}
	 \mathcal{O}\bigl(\tvnorm{\P_1-\P_2}\bigr) + \mathcal{O}\bigl(|\lb_1 -\lb_2|\bigr) +
	\mathcal{O}\bigl(\sup_{x\in\cX} \inorm{k_2(\cdot,x) - k_1(\cdot,x)} \bigr).
	\ee
\end{corollary}

Theorem \ref{sec1.thm2} and Corollary \ref{sec1.cor1} establish the upper bounds with respect to the supremum norm without any assumptions on the unknown
probability measures $\P_1, \P_2$ and the kernels $k_1, k_2$ besides continuity and boundedness. However, Theorem \ref{sec1.thm2} and Corollary \ref{sec1.cor1} unfortunately exclude the case of small values beyond $\min\{\lb_1, \lb_2\}.$ The next theorem shows a similar result for the case of the norm in a reproducing kernel Hilbert space, but for all $\lb_1, \lb_2>0.$

Provided some prior knowledge on RKHSs $H_1$ and $H_2$ is available, we assume
\be\label{sec1.inclusion}
H_2\subseteq H_1.
\ee
Then we can show that a similar total stability theorem to Theorem \ref{sec1.thm2} holds true in terms of $H_1$-norm. Please note that the following result holds true for \emph{all} positive $\lb_1$ and $\lb_2,$ which is contrasted to Theorem \ref{sec1.thm2}.

\begin{theorem}\label{sec1.thm3}
Let Assumptions \ref{sec2.assumption-spaces1}, \ref{sec2.assumption-kernel1}, and \ref{sec1.assumption2} be satisfied. Assume that $H_1$ and $H_2$ satisfy (\ref{sec1.inclusion}).
\bnum
\item If additionally the loss function $L$ is differentiable, then, for all $\lb_1,\lb_2>0,$ we have
\beqnal\label{sec1.hnorm1}
& &\hanorm{\faaa-\fbbb}\nonumber\\
&\leq & c_1^\prime(L, \lb_1, \lb_2)\cdot\tvnorm{\P_1-\P_2}+c_2^\prime(L, \lb_1, \lb_2)\cdot|\lb_1 -\lb_2|\nonumber\\
& &+c_3^\prime(L, \lb_1, \lb_2)\cdot\sup_{x\in \cX}\hanorm{\ka{x}-\kb{x}},\nonumber
\eeqnal
where $c_1^\prime(L, \lb_1, \lb_2):=\frac{\kappa|L|_1}{\min\{\lb_1, \lb_2\}}, c_2^\prime(L, \lb_1, \lb_2):=\frac{\kappa|L|_1}{\min\{\lb_1^2, \lb_2^2\}},$ and $c_3^\prime(L, \lb_1, \lb_2):=\frac{|L|_1}{2\min\{\lb_1, \lb_2\}}.$

\item Assume that the loss function $L(x, y, t)$  can be represented by a margin-based loss function $\tilde{L}(yt)$ for classification or by a distance-based loss function $\tilde{L}(y-t)$ for regression, which are convex and Lipschitz continuous with Lipschitz constant $|\tilde{L}|_1\in (0, \infty).$  Then for all $\lb_1,\lb_2>0,$ we have
\beqnal\label{sec1.hnorm1}
& &\hanorm{\faaa-\fbbb}\nonumber\\
&\leq & \tilde{c}_1^\prime(\tilde{L}, \lb_1, \lb_2)\cdot\tvnorm{\P_1-\P_2}+\tilde{c}_2^\prime(\tilde{L}, \lb_1, \lb_2)\cdot|\lb_1 -\lb_2|\nonumber\\
& &+\tilde{c}_3^\prime(\tilde{L}, \lb_1, \lb_2)\cdot\sup_{x\in \cX}\hanorm{\ka{x}-\kb{x}},\nonumber
\eeqnal
where $\tilde{c}_1^\prime(\tilde{L}, \lb_1, \lb_2):=\frac{\kappa|\tilde{L}|_1}{\min\{\lb_1, \lb_2\}}, \tilde{c}_2^\prime(\tilde{L}, \lb_1, \lb_2):=\frac{\kappa|\tilde{L}|_1}{\min\{\lb_1^2, \lb_2^2\}},$ and $\tilde{c}_3^\prime(\tilde{L}, \lb_1, \lb_2):=\frac{|\tilde{L}|_1}{2\min\{\lb_1, \lb_2\}}.$
\enum
\end{theorem}

\begin{remark}
There are many popular margin-based loss functions for classification and distance-based loss functions for regression satisfying Assumption \ref{sec1.assumption2}. These include the hinge loss $L_{\text{hinge}}$ and the logistic loss $L_{\text{c-logist}}$ for classification, the $\epsilon$-insensitive loss $L_{\epsilon\text{-insens}}$, the Huber's loss $L_{\a\text{-Huber}},$ and the logistic loss $L_{\text{r-logist}}$ for regression, the pinball loss $L_{\tau\text{-pin}}$ for quantile regression, which are defined before. All these loss functions are Lipschitz continuous with Lipschitz constants $|L_{\text{hinge}}|_1=1, |L_{\text{c-logist}}|_1=1, |L_{\epsilon\text{-insens}}|_1=1, |L_{\a\text{-Huber}}|_1=\a, |L_{\text{r-logist}}|_1=1,$ and $|L_{\tau\text{-pin}}|_1=\max\{\tau, 1-\tau\}\in(0, 1),$ respectively.

\end{remark}

\section{Results for pairwise learning}\label{sec2}

Let $(\cX,\ca A)$ be a measurable space and $\cY\subset \R$ be closed.
A function
\be \label{sec2.def-u2loss}
L:(\cXY)^2\times\R^2 \to [0,\infty)
\ee
is called a \textbf{pairwise loss function}, or simply
a \textbf{pairwise loss}, if it is measurable.
A pairwise loss  $L$ is \textbf{represented} by $\rho$, if $\rho:\R\to[0,\infty)$ is
a measurable function and, for all
$(x,y)\in\cXY$, for all $(\tix,\tiy)\in\cXY$, and for all $t,\tit\in\R$,
\begin{equation}\label{loss:def-mee}
L(x,y,\tix,\tiy,t,\tit) := \rho\bigl((y-t) - (\tiy-\tit) \bigr).
\end{equation}

\begin{definition}\label{sec2.loss:properties}
A pairwise loss  $L$ is called
\begin{enumerate}
\item \myem{(strictly) convex}, \myem{continuous}, or \myem{differentiable},
  if $L(\xyxy,\,\cdot\,,\,\cdot\,):\R^2\to [0,\infty)$
  is (strictly) convex,
  continuous, or (total) differentiable for all $(\xyxy)\in (\cXY)^2$, respectively.
\item \myem{locally separately Lipschitz continuous}, if
  for all $b\geq 0$ there exists a constant $c_{b}\geq 0$ such that,
  for all $t,\tit,t',\tit'\in [-b,b]$, we have
  \begin{equation}\label{loss:locallipschitz-loss-def}
  \sup_{\substack{x,\tix\in \cX\\ y,\tiy\in \cY}}
  \bigl| L(x,y,\tix,\tiy, t,\tit) - L(x,y,\tix,\tiy,t',\tit')\bigr| \,
  \leq \, c_{b} \, \bigl(|t-t'| + |\tit-\tit'|\bigr)\,.
\end{equation}
  Moreover, for $b\geq 0$, the smallest such constant $c_b$ is
  denoted by $|L|_{b,1}$.
  Furthermore, $L$ is called \myem{separately Lipschitz continuous}\footnote{We mention that \citet{Rio2013}
  used the related term ``separately 1-Lipschitz'' in a different context.}, if
  there exists a minimal constant $|L|_1\in[0,\infty)$ such that,
  for all $t,\tit,t',\tit'\in\R$,
   {(\ref{loss:locallipschitz-loss-def})}
  is satisfied, if we replace $c_b$ by $|L|_1$.
\end{enumerate}
\end{definition}

It is essential to have a valid definition of the kernel method for \emph{all}
probability measures on $\cXY$, even if $\cX$ and/or $\cY$ are unbounded.
To avoid any moment conditions, which are necessary for example for the case of the least squares loss function, we will need the following notion of shifted pairwise loss functions. Let $L$ be a pairwise loss function. Then the corresponding
\textbf{shifted pairwise loss function} (or simply the shifted version of $L$)
is defined by
\begin{eqnarray} \label{sec2.shiftedloss}
   & & \Ls:(\cXY)^2\times\R^2\to\R, \\
   & & \Ls(x,y,\tix,\tiy,t,\tit)  :=  L(x,y,\tix,\tiy,t,\tit)-L(x,y,\tix,\tiy,0,0). \label{sec2.shiftedloss2}
\end{eqnarray}
We adopt the definitions of continuity, (locally) separately Lipschitz continuity, and
differentiability of $\Ls$ from the same definitions for $L$, i.e. these properties are meant to be valid for the last two arguments, when the first four arguments are arbitrary but fixed.
We define the $\Ls$-risk, the regularized $\Ls$-risk, and the \RPLM based on $\Ls$ by
\beq
\RP{\Ls}{f} & := & \Ex_{\P^2} \Ls(X,Y,\tiX,\tiY,f(X),f(\tiX)) \\
\RPreg{\Ls}{f} & := & \RP{\Ls}{f} + \lb\hhnorm{f}\\
\fLsPk & := & \arg \inf_{f\in H} \RPreg{\Ls}{f} \,,
\eeq
respectively.
There exists a strong connection between $L$ and $\Ls$ in terms of convexity and separate Lipschitz continuity and also for the corresponding risks, see
\citet[Lemma B.8 to Lemma B.11]{ChristmannZhou2016a}.
Of course, shifting the loss function $L$ to $\Ls$ changes the objective function, but
the \emph{minimizers} of $\RPreg{L}{\cdot}$ and $\RPreg{\Ls}{\cdot}$ coincide for those
$\P\in\PM(\cXY)$ for which $\RPreg{L}{\cdot}$
has a minimizer in $H$, i.e.,  we have
\be \label{sec2.fLsPequalsfP}
\arg \inf_{f\in H} \RPreg{\Ls}{f} = \arg \inf_{f\in H} \RPreg{L}{f}, \mbox{\qquad if~} \fP\in H \mbox{~exists}.
\ee
Furthermore, {(\ref{sec2.fLsPequalsfP})} is valid for all empirical distributions $\D$ based on a data set consisting
of $n$ data points $(x_i,y_i)$, $1 \le i \le n$, because $f_{\D,\lb,k}$ exists and is unique since $\RT{L}{0}<\infty$.

\begin{assumption}\label{sec2.assumption-loss1}
Let $L$ be a separately Lipschitz-continuous, differentiable
convex pairwise loss function for which all
partial derivatives up to order 2 with respect to the last two arguments are
continuous and uniformly bounded in the sense that there exist
constants $c_{L,1}\in(0,\infty)$ and $c_{L,2}\in(0,\infty)$ with
\beq
\sup_{x,\tix\in \cX, ~ y,\tiy\in \cY} ~ | D_i L(x,y,\tix,\tiy,\,\cdot,\,\cdot\,) | & \le & c_{L,1}\, , \qquad i\in\{5,6\}, \label{loss-assump1}\\
\sup_{x,\tix\in \cX, ~ y,\tiy\in \cY} ~ | D_i D_j L(x,y,\tix,\tiy,\,\cdot,\,\cdot\,) |
& \le & c_{L,2} \, , \qquad i,j\in\{5,6\} . \label{loss-assump2}
\eeq
Let the partial derivatives $D_i \Ls$, $i\in\{5,6\}$, be uniformly Lip\-schitz continuous with
Lip\-schitz constants $|D_i L|_1$.
Additionally, assume that $L(x,y,x,y,t,t)=0$ for all $(x,y,t)\in\cX\times\cY\times\R$.
\end{assumption}


We can now state our total stability theorem for kernel based pairwise learning methods.

\begin{theorem}\label{sec2.thm1}
Let Assumptions \ref{sec2.assumption-spaces1},
\ref{sec2.assumption-kernel1}, and \ref{sec2.assumption-loss1} be satisfied.
Define $d_L:=|D_5 \Ls|_1 + |D_6 \Ls|_1$.
If
$\min\{\lb_1, \lb_2\} > \kappa^2 \cdot d_L$, then
\beq
  & & \inorm{\faaa - \fbbb} \nonumber \\
  & \le & C_1(L)  \cdot \tvnorm{\P_1-\P_2}
          +  C_2(L) \cdot |\lb_1 - \lb_2|
          +  C_3(L,\lb_1,\lb_2) \cdot \sup_{x\in\cX} \bigl(\inorm{k_2(\cdot,x) - k_1(\cdot,x)} \bigr) \, , \nonumber
\eeq
where
$$
  C_1(L):=\frac{4c_{L,1}}{d_L},~~
  C_2(L):=\frac{|L|_1}{d_L}, \text{~~and~~}
  C_3(L,\lb_1,\lb_2):= \frac{c_{L,1}}{\min\{\lb_1, \lb_2\} - \kappa^2 d_L}.
$$
\end{theorem}

\begin{corollary}\label{sec2.cor1}
Let the assumptions of Theorem \ref{sec2.thm1} be satisfied and the constants
$C_1(L)$, $C_2(L)$, and $C_3(L,\lb_1,\lb_2)$ be defined in the same manner. If there exists a constant $s$ such that
$0 < s < \min\{\lb_1,\lb_2\} - \kappa^2 d_L$,
then
\beq
  & & \inorm{\faaa - \fbbb}  \label{sec2.cor1F1}\\
  & \le & C_1(L)  \cdot \tvnorm{\P_1-\P_2}
          +  C_2(L) \cdot |\lb_1 - \lb_2|
          +  \bar{C}_3(L) \cdot \sup_{x\in\cX} \bigl(\inorm{k_2(\cdot,x) - k_1(\cdot,x)} \bigr) \, , \nonumber
\eeq
and
\beq
  & & \big|\Lsriska{\fsaaa}-\Lsriskb{\fsbbb}\big| \label{sec2.cor1F2}\\
  & \le &    C_4(L) \cdot \tvnorm{\P_1-\P_2}
          +  C_5(L) \cdot |\lb_1 - \lb_2|
          +  C_6(L) \cdot \sup_{x\in\cX} \bigl(\inorm{k_2(\cdot,x) - k_1(\cdot,x)} \bigr) \, , \nonumber
\eeq
where
$$
\bar{C}_3(L): = \frac{c_{L,1}}{s},~~
  C_4(L):=\frac{4 |L|_1 (|L|_1 + 2 c_{L,1})}{d_L},~~
  C_5(L):=\frac{2|L|_1^2}{d_L}, \text{~~and~~}
  C_6(L,\lb_1,\lb_2):= \frac{2 |L|_1 c_{L,1}}{s}.
$$
Therefore, both terms $\inorm{\faaa - \fbbb}$ and $\big|\Lsriska{\fsaaa}-\Lsriskb{\fsbbb}\big|$ are of the order
\be \label{sec2.ThreeBigOs}
\mathcal{O}\bigl(\tvnorm{\P_1-\P_2}\bigr) + \mathcal{O}\bigl(|\lb_1 -\lb_2|\bigr) +
  \mathcal{O}\bigl(\sup_{x\in\cX} (\inorm{k_2(\cdot,x) - k_1(\cdot,x)}) \bigr).
\ee
\end{corollary}

In other words, we have relatively simple, but explicit upper bounds of
$\inorm{\faaa - \fbbb}$ and $\big|\Lsriska{\fsaaa}-\Lsriskb{\fsbbb}\big|$ and these upper bounds are a \emph{weighted sum}
--with known and non-stochastic weights-- of
\bnum
\item the (norm of total variation) distance between the measures $\P_1$ and $\P_2$,
\item the distance between the regularization parameters $\lb_1$ and $\lb_2$, and
\item the supremum norm of the canonical feature maps
$k_1(\cdot,x)$ and $k_2(\cdot,x)$, $x\in \cX$, of the kernels.
\enum

Theorem \ref{sec2.thm1} and Corollary \ref{sec2.cor1} have two advantages: (i) there are no assumptions on the unknown probability measures $\P_1,\P_2$ and on the kernels $k_1, k_2$;
(ii) the upper bounds are with respect to the supremum norm.
A special case occurs for example if both probability measures are empirical measures
belonging to data sets $D_1\in(\cXY)^{n_1}$ and $D_2\in(\cXY)^{n_2}$, respectively, where $n_1$ and $n_2$ denote the sample sizes.
However, an obvious disadvantage of Theorem \ref{sec2.thm1} and Corollary \ref{sec2.cor1} is that
values of $\min\{\lb_1,\lb_2\}$ close to zero are excluded. The next result shows a similar result for the case that the supremum norm is replaced by the norm in a reproducing kernel Hilbert space.

%

\begin{theorem}\label{sec2.thm2}
Let Assumptions \ref{sec2.assumption-spaces1}, \ref{sec2.assumption-kernel1} be satisfied. Assume that $H_1$ and $H_2$ satisfy (\ref{sec1.inclusion}).
\bnum
\item If additionally the pairwise loss function $L$ satisfies Assumption \ref{sec2.assumption-loss1}, then for all $\lb_1,\lb_2>0,$ we have
\beqnal\label{sec2.hnorm1}
	& &\hanorm{\faaa-\fbbb}\nonumber\\
	&\leq & C_1^\prime(L, \lb_1, \lb_2)\tvnorm{\P_1-\P_2}+C_2^\prime(L, \lb_1, \lb_2)|\lb_1 -\lb_2|+C_3^\prime(L, \lb_1, \lb_2)\sup_{x\in \cX}\hanorm{\ka{x}-\kb{x}},\nonumber
	\eeqnal
	where $C_1^\prime(L, \lb_1, \lb_2):=\frac{4\kappa c_{L, 1}}{\min\{\lb_1, \lb_2\}}, C_2^\prime(L, \lb_1, \lb_2):=\frac{\kappa|L|_1}{\min\{\lb_1^2, \lb_2^2\}},$ and $C_3^\prime(L, \lb_1, \lb_2):=\frac{c_{L, 1}}{\min\{\lb_1, \lb_2\}}.$

\item Assume that the pairwise loss function $L$ can be represented by a convex and Lipschitz continuous function $\rho:\R\to\R$ with Lipschitz constant $|\rho|_1,$ see {(\ref{loss:def-mee})}, i.e. we have
$L(x,y,\tix,\tiy,t,\tit) := \rho\bigl((y-t) - (\tiy-\tit) \bigr)$. Then for all $\lb_1,\lb_2>0,$ we have
\beqnal\label{sec2.hnorm1}
& &\hanorm{\faaa-\fbbb}\nonumber\\
&\leq & \tilde{C}_1^\prime(\rho, \lb_1, \lb_2)\tvnorm{\P_1-\P_2}+\tilde{C}_2^\prime(\rho, \lb_1, \lb_2)|\lb_1 -\lb_2|+\tilde{C}_3^\prime(\rho, \lb_1, \lb_2)\sup_{x\in \cX}\hanorm{\ka{x}-\kb{x}},\nonumber
	\eeqnal
	where $\tilde{C}_1^\prime(\rho, \lb_1, \lb_2):=\frac{4\kappa |\rho|_1}{\min\{\lb_1, \lb_2\}}, \tilde{C}_2^\prime(\rho, \lb_1, \lb_2):=\frac{\kappa|\rho|_1}{\min\{\lb_1^2, \lb_2^2\}},$ and $\tilde{C}_3^\prime(\rho, \lb_1, \lb_2):=\frac{|\rho|_1}{\min\{\lb_1, \lb_2\}}.$	
\enum
\end{theorem}

We will give some examples when $\tvnorm{\P_1-\P_2}$ and
$\sup_{x\in\cX} (\inorm{k_2(\cdot,x) - k_1(\cdot,x)})$ are small in the next
Section.

\section{Examples}\label{sec3}
  Let us first consider some simple conditions, when the norm of total variation $\tvnorm{\P_1-\P_2}$
between two probability measures is small.
Let $\P,\P_n$, $n\in\N$, be probability measures on the same measurable space $(\O,\sA)$, where $(\O,d_\O)$
is a separable metric space.
It is well-known that $\tvnorm{\P_n-\P} \to 0$, if $n\to\infty$, implies that the Prohorov metric $\dPro(\P_n,\P) \to 0$ and the latter is equivalent to the weak convergence of $(\P_n)_{n\in\N}$ to
$\P$, see e.g. \citet[p.\,34]{Huber1981} or \citet[Thm.\,11.3.3]{Dudley2002}.
Furthermore, if $\P_n$ and $\P$ have densities $f_n$ and $f$ with respect to some
$\sigma$-finite measure $\nu$ on $(\O,\sA)$, then $f_n \to f$ $\nu$-a.s. implies
$\tvnorm{\P_n-\P}\to 0$, because by Scheff{\'e}'s
theorem $\sup_{A\in\sA} |\P_n(A)-\P(A)| \le \int |f_n-f| \,d\nu \to 0$,
see e.g. \citet[p.\,29]{Billingsley1999}.

Let us now consider the typical situation, when a researcher asks himself what would happen
if at most $\ell$ data points of the original data set $D_1=\bigl( (x_1,y_1),\ldots,(x_n,y_n)\bigr)$ can be extreme outliers or may be changed in an arbitrary manner.
Let us denote a second data set by $D_2=\bigl( (\tix_{n+1},\tiy_{n+1}),\ldots,(\tix_N,\tiy_N)\bigr)$, where $N=n+n$. We assume that at most $\ell$ data points contained in $D_2$ differ from the data points in $D_1$.
Then we can write the corresponding empirical measures as
$\P_1:= \sum_{i=1}^N w_i \delta_{(x_i,y_i)}$ and
$\P_2:= \sum_{i=1}^N \tilde{w}_i \delta_{(x_i,y_i)}$,
where all weights satisfy $w_i,\tilde{w}_i \in[0,1]$ and $\sum_i w_i = \sum_i \tilde{w}_i=1$.
Because $D_1$ and $D_2$ differ by at most $\ell$ data points, we obtain
$\tvnorm{\P_1-\P_2} \le \frac{\ell}{n}$.
Of course, a similar argumentation is possible for the case, that at most $\ell$ arbitrarily chosen data points can be added to the original data set $D_1$.

Let us now consider the case under which conditions $\sup_{x\in\cX} \bigl(\inorm{k_2(\cdot,x) - k_1(\cdot,x)} \bigr)$ is small. We can upper bound $\sup_{ x\in\cX}\inorm{\kb{x}-\ka{x}}$ explicitly for some special kernels. Here we take Gaussian RBF kernels, Sobolev kernels, and hierarchical Gaussian kernels as our examples.


It is well-known that a differentiable function $g:\R\to\R$ is
Lipschitz continuous with Lipschitz constant
$|g|_{1}=\sup_{x\in\R} |g'(x)|$ if and only if $g$ has a uniformly bounded derivative $g'$.
It is easy to see that the set
\be
  \mathcal{G}:=\bigl\{ g:\R\to\R; ~g \mbox{~is~differentiable~and~} g' \mbox{~is~uniformly~bounded} \bigr\}
\ee
equipped with the binary operations $+$ and $\cdot$ is a 
commutative ring: Define $0$ and $1$ as the constant functions with values always equal to $0$ and $1$, respectively.
We have, for all $g,g_{1}, g_{2}, g_{3}\in\mathcal{G}$,
\begin{eqnarray*}
  (g_{1}+g_{2})+g_{3}=g_{1}+(g_{2}+g_{3}), & g_{1}+g_{2}=g_{2}+g_{1}, & 0\in\mathcal{G}, g+0=g,\\
  (-g)+g=0, &  (g_{1}\cdot g_{2})\cdot g_{3}=g_{1}\cdot (g_{2}\cdot g_{3}), &  1\in\mathcal{G},  1\cdot g = g \cdot 1 = g,\\
  g_{1}\cdot(g_{2}+g_{3})=g_{1}\cdot g_{2} + g_{1}\cdot g_{3}, &
  1 \ne 0, & g_{1} \cdot g_{2} = g_{2} \cdot g_{1}.
\end{eqnarray*}
Obviously, $g\in\mathcal{G}$ implies that, for all $\gamma\in(0,\infty)$,
the function $\frac{1}{\gamma} g \in \mathcal{G}$, too.
Furthermore, if $g_{1},g_{2}\in \mathcal{G}$, then $g_{2}\circ g_{1}\in\mathcal{G}$.

A special case of such a function in $\mathcal{G}$ is $\phi$ given by
\be \label{sec3a.F1}
\phi:\R\to\R, \quad \phi(r)=h(|r|),
\ee
where $h: [0,\infty)\to\R$ is supported on $[0, c]$ for some constant $c\in (0, \infty]$ and
$h'$ is uniformly bounded with $h'_{+}(0)=0$.

Many RBF kernels $k$ on $\cX\subset\Rd$ generated by $\phi$ are of the form {(\ref{sec3a.F1})}, e.g.
Gaussian RBF kernels with $c=\infty$ and
$h(r)=\exp(-|r|^{2})$, $r \in [0,\infty)$.
Another example is a radial basis kernel with compact support, where
$c=1$, $h(r)=\phi_{d,m}(|r|)$,
and $\phi_{d,m}$ is a certain univariate polynomial $p_{d,m}$ of degree $\lfloor d/2  \rfloor + 3m +1$ for $m\in\N$ and the RKHSs of these
kernels are special Sobolev spaces, see
 \citet{Wu1995}, \citet{Wendland1995}, and \citet[Thm.\,9.13, Thm.\,10.35]{Wendland2005} for details.
For simplicity, we exclude the case $\phi_{d,0}$ which yields non-differentiable functions.

Let $k(x,x')=\phi(\snorm{x-x'})$ be a kernel on a bounded set $\cX\subset\Rd$,
where $\phi$ satisfies {(\ref{sec3a.F1})} and $h$ satisfies {(\ref{sec3a.F1})}.
It follows from \citet[Lemma 4.3]{SC2008} that, for all
$\gamma\in(0,\infty)$,
\be \label{sec3a.F3}
  k_{\gamma}:\cX\times\cX\to\R, ~
  k_{\gamma}(x,x') = \phi(\snorm{x-x'}/\gamma)
\ee
is a kernel on $\cX$, too.
Fix $a\in(0,\infty)$.
Let $0<a \le\gamma_{1}\le\gamma_{2}<\infty$.
We easily see that the Lipschitz continuity of $\phi$ implies
\beq
  \sup_{x\in\cX} \inorm{k_{\gamma_{1}}(\cdot,x)-k_{\gamma_{2}}(\cdot,x)}
  & = &  \sup_{x,x'\in\cX} \Big| \phi\Bigl(\frac{\snorm{x-x'}}{\gamma_{1}}\Bigr) - \phi\Bigl(\frac{\snorm{x-x'}}{\gamma_{2}}\Bigr) \Big| \\
  & \le & \sup_{x,x'\in\cX} \Bigl( |\phi|_{1} \cdot \snorm{x-x'} \cdot \Bigl|\frac{1}{\gamma_{1}} - \frac{1}{\gamma_{2}}\Bigr| \Bigr)  \\
  & = &   \frac{|\phi|_{1}}{\gamma_{1} ~ \gamma_{2}} (\gamma_{2}-\gamma_{1}) ~ \sup_{x,x'\in\cX}\snorm{x-x'}  \\
  & \le &   \frac{|\phi|_{1} \diam(\cX)}{a^{2}} (\gamma_{2}-\gamma_{1}).
\eeq
Therefore, we obtain under these conditions that
\be
  \sup_{x\in\cX} \inorm{k_{\gamma_{1}}(\cdot,x)-k_{\gamma_{2}}(\cdot,x)}
  =
  \mathcal{O}(|\gamma_{1}-\gamma_{2}|).
\ee

Hence Gaussian RBF kernels $k_{\gamma_{1}}, k_{\gamma_{2}}$ and the above mentioned compactly supported RBF kernels satisfy the condition that
$$ \sup_{x\in\cX} \inorm{k_{\gamma_{1}}(\cdot,x)-k_{\gamma_{2}}(\cdot,x)}=  \mathcal{O}(|\gamma_{1}-\gamma_{2}|).$$


\citet{SteinwartThomannSchmid2016} introduced hierarchical Gaussian kernels, which highlights the similarities to deep architectures in deep learning (see \citet{Goodfellow-et-al-2016}). The hierarchical Gaussian kernels are constructed iteratively with composing weighted sums of Gaussian kernels in each layer. We call kernels $k_{\g, \cX, H }$ of the following form hierarchical Gaussian kernels
\be\label{sec1.hierarchicalGaussian}
k_{\g, \cX, H }=\exp\big(-\g^{-2}\|\kk{x}-\kk{x'}\|_H^2\big)
\ee
where $k$ is a kernel generating RKHS $H$ and $\|\kk{x}-\kk{x'}\|_H^2=k(x, x)-2k(x, x')+k(x', x').$

To investigate how hierarchical Gaussian kernels detect the deep architectures in deep learning, we introduce some notations first. For $x=(x_1, \ldots, x_d)\in \cX\subset\Rd$ and $I\subset\{1, \ldots, d\},$ let
$$x_{I}=(x_i)_{i\in I}$$
be the vector projected onto the coordinates listed in $I$ and $\cX_I=\{x_I: x\in \cX\}.$ Assume that we have non-empty sets $I_1, \ldots, I_\ell\subset \{1, \ldots, d\}$ and some weights $w_1, \ldots, w_\ell>0,$ as well kernels $k_i$ on $\cX_{I_i}$ for all $i=1, \ldots, \ell.$ For $I=I_1\cup\cdots\cup I_\ell,$ we then define a new kernel on $\cX_I$ by
\be\label{sec1.weightkernel}
k(x, x')=\sum_{i=1}^{\ell}w_i^2k_i(x_{I_i}, x'_{I_i}), \quad x, x'\in \cX_I.
\ee
The following definition considers iterations of (\ref{sec1.weightkernel}).
\begin{definition}
Let $k$ be a kernel of the form (\ref{sec1.weightkernel}) and $H$ be its RKHS. Then the resulting hierarchical Gaussian kernel $k_{\g, \cX_I, H }$ is said to be of depth
\begin{itemize}
\item $m=1,$ if all kernels $k_1, \ldots, k_\ell$ in (\ref{sec1.weightkernel}) are linear.
\item $m>1,$ if all $k_1, \ldots, k_\ell$ in (\ref{sec1.weightkernel}) are hierarchical kernels of depth $m-1.$
\end{itemize}
\end{definition}

For any $ x=(x_1, \cdots, x_d)\in \cX, x'=(x_1^\prime, \cdots, x_d^\prime)\in \cX,$ the hierarchical kernels of depth $1$ with $I_i=\{i\}$ are of the form
\be\label{sec1.depth1}
k_{\bf{w}, \g_1}(x, x')=\exp\big(-2\g_{1}^{-2}\sum_{i\in I}w_i^2 (x_i-x_i^{\prime})^2\big)
\ee
for some suitable ${\bf w}=(w_i)_{i\in I}$ with $w_i>0$ for all $i\in I,$ and $\g_1>0.$ We call these kernels of form (\ref{sec1.depth1}) \emph{inhomogeneous Gaussian kernels} contrasted to the standard Gaussian kernels.

To derive an explicit formula for depth $2$ kernels, we fix some $I_1, \ldots, I_\ell\subset  \{1, \ldots, d\},$ some first layer weight vectors ${\bf w}_1=(w_{1, j})_{j\in I_1}, \ldots, {\bf w}_\ell=(w_{\ell, j})_{j\in I_\ell}$ and second layer weight vector ${\bf w}=(w_1, \ldots, w_\ell).$ Writing ${\bf W}^{(1)}:=({\bf w}_1, \ldots, {\bf w}_\ell),$ the hierarchical Gaussian kernel $k_{{\bf W}^{(1)}, {\bf w}, \g_1,\g_2}$ of depth 2 with $\g_1>0, \g_2>0$ that is built upon the kernels $k_{{\bf w}_1, \g_1}, \ldots, k_{{\bf w}_\ell, \g_1}$ and weights ${\bf w}=(w_1, \ldots, w_\ell)$ is given by
$$k_{{\bf W}^{(1)}, {\bf w}, \g_1, \g_2}(x, x')
=\exp\Big( -2\g_{2}^{-2}\sum_{i=1}^{\ell}w_i^2(1-k_{{\bf w}_i, \g_1}(x_{I_i}, x'_{I_i}))\Big).$$

Repeating the similar calculations and denoting ${\bf W}^{(j)}=({\bf W}^{(j)}_1, \cdots, {\bf W}^{(j)}_\ell)$, we see that hierarchical kernel of depth $m\geq 3$ with $\g_1>0, \cdots, \g_m>0$ is given by
$$k_{{\bf W}^{(1)}, \ldots, {\bf W}^{(m-1)}, {\bf w}, \g_1, \cdots, \g_{m}}(x, x')=\exp\Big( -2\g_m^{-2}\sum_{i=1}^{\ell}w_i^2\big(1-k_{{\bf W}^{(1)}_i,\ldots, {\bf W}^{(m-2)}_i, {\bf w}^{(m-1)}_i, \g_1,\cdots, \g_{m-1}}(x_{I_i}, x'_{I_i})\big)\Big),$$
where $k_{{\bf W}^{(1)}_i,\ldots, {\bf W}^{(m-2)}_i, {\bf w}^{(m-1)}_i, \g_1,\cdots, \g_{m-1}}$ denote hierarchical kernels of depth $m-1.$

The next result shows that the hierarchical Gaussian kernels only vary a little if the parameters involving in the form slightly change.
\begin{corollary}\label{sec3.corollary2}
Assume $\cX\subset\Rd$ is bounded, i.e., $diam(\cX)<\infty.$
\bnum
\item For depth $m=1$ hierarchical Gaussian kernels $k_{\bf{w}, \g_1}$ and $k_{\tilde{\bf{w}}, \g_1}$ with different parameters $\bf{w}$ and $\tilde{\bf{w}}$ satisfying $\sum_{i\in I} w_i^2 \leq 1$ and $\sum_{i\in I} \tilde{w}_i^2 \leq 1,$ we have
\beq\label{depth1b}
\sup_{x\in\cX} \inorm{k_{\bf{w}, \g_1}(\cdot,x)-k_{\tilde{\bf{w}}, \g_1}(\cdot,x)}= \mathcal{O}\big(\|\bf{w}-\tilde{\bf{w}}\|_{\ell^2}\big).
\eeq

\item For depth $m>1$ assume $\|{\bf W}^{(j)}\|_{\ell_2}:=\sum_{i=1}^{\ell}\|{\bf w}^{(j)}_i\|_{\ell^2} \leq 1$ by scaling $\g_j$, $j=1, \cdots, m, $
where ${\bf w}^{(m)} = {\bf w}$. Then we have
\beq\label{depthmb}
& &\sup_{x\in\cX} \inorm{k_{{\bf W}^{(1)},\cdots, {\bf W}^{(m-1)}, {\bf w}, \g_1,\cdots, \g_m}(\cdot,x)-k_{\tilde{\bf W}^{(1)}, \cdots, \tilde{\bf W}^{(m-1)}, \tilde{\bf w}, \g_1,\cdots, \g_m}(\cdot,x)}\nonumber\\
&=&\mathcal{O}\Big(\|\mathbf{w}-\tilde{\mathbf{w}}\|_{\ell^2}+\sum_{j=1}^{m-1}\|{\mathbf W}^{(j)} - \tilde{{\mathbf W}}^{(j)}\|_{\ell^2} \Big).
\eeq
\enum
\end{corollary}

\section{Discussion}\label{discussion}
This paper established some results on the total stability of a class of regularized empirical risk minimization methods based on kernels. We showed that such methods are not only robust with respect to small
variations of the distribution $\P$, but that these methods are even totally stable if the full triple $(\P,\lb,k)$ consisting of the -- of course completely unknown -- underlying probability measure $\P$, the regularization parameter $\lambda$, and the kernel $k$ or its RKHS $H$ changes slightly. 
Let us denote by $f_{\P_{j},\lb_{j},k_{j}}$ the regularized empirical risk minimization method for the triple $(\P_{j},\lb_{j},k_{j})$, $j\in\{1,2\}$.  The main results show that the difference of both methods, i.e.
$$
\snorm{f_{\P_{1},\lb_{1},k_{1}}-f_{\P_{2},\lb_{2},k_{2}}}_{\bullet} 
$$
(where $\snorm{\cdot}_{\bullet}$ denotes either the 
supremum norm or a Hilbert space norm),
depends on the norm of total variation between the probability measures 
$\P_{1}$ and $\P_{2}$,
on the difference $\lb_1 -\lb_2$ between the regularization parameters,
and on the difference of the kernels measured by the supremum norm of their canonical features maps, i.e. on
$\sup_{x\in\cX} (\inorm{k_2(\cdot,x) - k_1(\cdot,x)})$.
We derived upper bounds for $\snorm{f_{\P_{1},\lb_{1},k_{1}}-f_{\P_{2},\lb_{2},k_{2}}}_{\bullet}$ with explicit constants.

We did not address the following questions which are beyond 
the scope of this paper.
It would be interesting to have modifications of our results for the supremum norm without assumptions on the regularization parameters. We obtained such results for the case of Hilbert space norms, see e.g. Theorem \ref{sec1.thm3} and Theorem \ref{sec2.thm2}. We conjecture that other techniques are needed to obtain such results for the supremum norm.
Furthermore, it seems to be an open question whether learning with hierarchical Gaussian RBF kernels is in general even stable if the depth $m$ and the index sets are not fixed.

Finally, we would like to mention that it is possible to use the well-known identity
$f_{\P,\lb,k} = f_{\P,1, (k/\lb)}$. 
We computed similar upper bounds than the ones given in this paper for 
$$
  \inorm{f_{\P_{1},\lb_{1},k_{1}} - f_{\P_{2},\lb_{2},k_{2}}} =
  \inorm{f_{\P_{1},1,(k_{1}/\lb_{1})} - f_{\P_{2},1,(k_{2}/\lb_{2})}},
$$
but the results were almost identical to the ones presented here, only the factor
belonging to $|\lb_{1}-\lb_{2}|$ can slightly change.


\section{Appendix}
\subsection{Appendix A: Some tools}\label{appendixa}

To improve the readability of the paper, let us first collect some properties of Bochner integrals.

\begin{lemma}\label{appendixc.lem1}
Let $(\O,\sA)$ be a measurable space, $H$ be a Hilbert space with norm $\hnorm{\cdot}$, and
let $f:(\O,\sA)\to (H,\B(H))$ be a measurable function.
\begin{enumerate}
\item If $\mu$ is a finite measure on $(\O,\sA)$, then
\be \label{appendixc.lem1f1}
  \Big\| \int f \, d\mu \Big\|_H
  \le
  \int \hnorm{f} \, d\mu.
\ee
\item Let $\mu$ be a finite signed measure on $(\O,\sA)$ and denote the total variation
measure by $|\mu|=\mu_+ + \mu_-$. Then
\be \label{appendixc.lem1f2}
  \Big\| \int f \, d\mu \Big\|_H
  \le
  \int \hnorm{f} \, d|\mu| .
\ee
\item Let $\P,\Q$ be probability measures on $(\O,\sA)$ and let
$g:(\O\times\O,\sA\otimes\sA)\to (H,\B(H))$ be a measurable function.  Then
\beq \label{appendixc.lem1f3}
  \Big\| \int g \, d\P^2  - \int g \, d\Q^2 \Big\|_H
  & \le &
  2 \, \hnorm{g} \, \tvnorm{\P-\Q},
\eeq
where $\tvnorm{\cdot}$ denotes the norm of total variation.
\end{enumerate}
\end{lemma}

\begin{proof}
We refer to \citet[p.46, Thm. 4 \textit{(ii)}]{DiestelUhl1977} for the first assertion.

We use the Hahn-Jordan decomposition of $\mu$ to prove the second assertion. We write $\mu$ as $\mu=\mu_+ - \mu_-$ and denote the
total variation measure by $|\mu|:=\mu_+ + \mu_-$. Then
\beq
  \Big\| \int f \,d\mu  \Big\|_H
  & = & \Big\| \int f \,d(\mu_+ - \mu_-)  \Big\|_H \nonumber
  =  \Big\| \int f \,d\mu_+  - \int f \,d\mu_-  \Big\|_H \nonumber \\
  & \le & \Big\| \int f \,d\mu_+ \Big\|_H + \Big\| \int f \,d\mu_-  \Big\|_H \nonumber \\
  & \stackrel{\scriptsize{(\ref{appendixc.lem1f1})}}{\le} &
    \int \hnorm{f} \,d\mu_+ + \int \hnorm{f} \,d\mu_-  \nonumber \\
  & = & \int \hnorm{f} \,d(\mu_+ + \mu_-)
   =  \int \hnorm{f} \,d|\mu| \, ,  \nonumber
\eeq
which yields the second assertion.

Let us now prove the third assertion.
\beq
  & &  \Big\| \int g \, d\P^2  - \int g \, d\Q^2 \Big\|_H \nonumber \\
  & \stackrel{\scriptsize{Fubini}}{=} & \Big\| \int \Big(\int g \,d\P - \int g\,d\Q \Big) \,d\P +
  \int \Big(\int g \,d\Q \Big) \,d(\P-\Q)  \Big\|_H \nonumber \\
  & \stackrel{\scriptsize{Fubini}}{=}  &  \Big\| \int \int g \,d\P \,d(\P-\Q) +
  \int \int g\,d\Q \,d(\P-\Q)   \Big\|_H \nonumber \\
  & = &  \Big\| \int \int g \,d(\P+\Q) \,d(\P-\Q) \Big\|_H \nonumber \\
  & \stackrel{\scriptsize{(\ref{appendixc.lem1f2})}}{\le} & \int \Big\| \int g \,d(\P+\Q) \Big\|_H
  \, d|\P-\Q| \nonumber \\
  & \stackrel{\scriptsize{(\ref{appendixc.lem1f2})}}{\le} & \int \int \hnorm{g} \,d(\P+\Q) \, d|\P-\Q| \nonumber \\
  & \le & 2 \, \hnorm{g} \, \tvnorm{\P-\Q} \,, \nonumber
\eeq
where we used in the last inequality that $\P$ and $\Q$ are probability measures.
\end{proof}


\begin{lemma}\label{appendixc.lem2}
If Assumptions \ref{sec2.assumption-spaces1},
\ref{sec2.assumption-kernel1}, and
\ref{sec2.assumption-loss1} are satisfied, then
\beq
   \inorm{\faaa - \fbaa}
   \le
   \frac{4}{\lb_1} c_{L,1} \inorm{k_1}^2 \  \tvnorm{\P_1-\P_2}.\nonumber
\eeq
\end{lemma}

\begin{proof}
To shorten the notation in the proof, we define $\lb:=\lb_1$, $k:=k_1$, and $H:=H_1$.
The well-known identity
$\inorm{f} \le \inorm{k} \hnorm{f}$ for all $f\in H$ yields
\be \label{appendixc.lem1f6}
 \inorm{\falk - \fblk} \le \inorm{k}  \hnorm{\falk - \fblk} \, .
\ee
Hence it suffices to derive an upper bound for $\hnorm{\falk - \fblk}$.
Recall that
\be
  \inorm{\kk{x}} \le \inorm{k} \cdot \hnorm{\kk{x}} \le \inorm{k}^2 \, , \qquad x\in\cX.\nonumber
\ee
For any choice of $(\P_1, \lb, k)$, the kernel estimator $\falk\in H$ exists and is unique. Define the function $g:(\cX\times\cY)^2\to H$ by
\beq
  & & g(x,y,\tix,\tiy) \nonumber \\
  & : = & D_5 \Ls(x,y,\tix,\tiy,\falk(x),\falk(\tix)) \kk{x}
       + D_6 \Ls(x,y,\tix,\tiy,\falk(x),\falk(\tix)) \kk{\tix}\,, \nonumber
\eeq
where $D_i \Ls$ denotes the partial derivative of $\Ls$ with respect to the $i^{th}$-argument,
$i\in\{5,6\}$. Of course $D_i L=D_i \Ls$.
Using {(\ref{loss-assump1})} and $\sup_{x\in\cX} \hnorm{\kk{x}} = \inorm{k}$, we obtain that
\be \label{appendixc.lem1f7}
  \hnorm{g} \le 2 \, c_{L,1} \, \inorm{k}.
\ee
Hence the representer theorem for pairwise loss functions, see Theorem \ref{sec2.representerthm}, yields
\beq
 \hnorm{\falk -\fblk}
 & \le & \frac{1}{\lb}
         \Big\| \int g \,d\P_1^2  - \int g \,d\P_2^2  \Big\|_H  \nonumber  \\
 & \stackrel{\scriptsize{(\ref{appendixc.lem1f3})}}{\le} &
          \frac{2}{\lb} \hnorm{g} \tvnorm{\P_1-\P_2}
    \stackrel{\scriptsize{(\ref{appendixc.lem1f7})}}{\le}   \frac{4}{\lb} c_{L,1} \inorm{k} \tvnorm{\P_1-\P_2} \, \nonumber
\eeq
which gives the assertion, if we combine this inequality with (\ref{appendixc.lem1f6}).
\end{proof}

\subsection{Appendix B: Proofs for results in Section \ref{sec1}}\label{appendixb}

The general representer theorem is a main tool in our proofs. We only need it for the case of differentiable loss function. Hence, in order that our paper is self-contained, we recall the theorem for differentiable loss function below, which is a special case of \citet[Thm.7]{ ChristmannVanMessemSteinwart2009}.

\begin{theorem}[\textbf{Representer theorem for SVMs}\label{sec1.representerthm}]
Let Assumptions \ref{sec2.assumption-spaces1}, \ref{sec2.assumption-kernel1}, and \ref{sec1.assumption2} be valid. We further assume that the loss $L$ is differentiable with respect to the third argument.
Then we have, for all $\P,\Q\in\PM(\cXY)$ and for all $\lb\in(0,\infty):$
\begin{enumerate}
\item The estimator $f_{\P, \lb, k}$ defined as the minimizer of
$\min_{f\in H}\{\Ex_{\P}\Ls(X, Y, f(X))+\lb\hhnorm{f}\}$ exists, is unique, and satisfies
\be \label{thm.representer.f1}
f_{\P, \lb, k}  =  - \frac{1}{2\lb} \Ex_{\P} ( h_{\P}(X, Y)\Phi(X)),    \nonumber
\ee
where $h_{\P}$ denotes the first derivative
\be
   h_{\P}(X, Y)   :=   L^{\prime}(X, Y, f_{\P, \lb, k} (X)). \nonumber  \ee
\item Furthermore,
\be
 \hnorm{ f_{\P, \lb, k} - f_{\Q, \lb, k}} \le \lb^{-1}\Big\| \Ex_{\P}(h_\P(X, Y) \Phi(X))-   \Ex_{\Q}(h_\P(X, Y) \Phi(X))  \Big\|_H \, . \nonumber
\ee
\end{enumerate}
\end{theorem}

\begin{proofof}{\textbf{Proof of Theorem \ref{sec1.thm1}}}
To shorten the notations in the proof, we define $f_\lb:=f_{\P, \lb, k}$ and $f_\mu:=f_{\P, \mu, k}.$ The representer theorem for SVMs, see Theorem \ref{sec1.representerthm}, tells us that
$$f_{\lb}-f_{\mu}=-\frac{1}{2\lb}\int_{\cX\times\cY}L^{\prime}(x, y, f_{\lb}(x))\Phi(x)\ d\P(x,y)+\frac{1}{2\mu}\int_{\cX\times\cY}L^{\prime}(x, y, f_{\mu}(x))\Phi(x)\ d\P(x,y).$$
Plugging the above formula into the RKHS norm of $f_{\lb}-f_{\mu},$ we get that
\beqnal\label{appendixb.difference}
\hhnorm{f_{\lb}-f_{\mu}}&=& \langle f_{\lb}-f_{\mu}, f_{\lb}-f_{\mu}\rangle_\cH\nonumber\\
&=&\Big\langle \frac{1}{2\mu}\int_{\cX\times\cY}L^{\prime}(x, y, f_{\mu}(x))\Phi(x)\ d\P(x,y),\ f_{\lb}-f_{\mu} \Big\rangle_\cH\nonumber\\
& & - \Big\langle \frac{1}{2\lb}\int_{\cX\times\cY}L^{\prime}(x, y, f_{\lb}(x))\Phi(x)\ d\P(x,y),\ f_{\lb}-f_{\mu} \Big\rangle_\cH\nonumber\\
&=& \frac{1}{2\mu}\int_{\cX\times\cY}L^{\prime}(x, y, f_{\mu}(x))(f_{\lb}(x)-f_{\mu}(x))\ d\P(x,y)\nonumber\\
& & - \frac{1}{2\lb}\int_{\cX\times\cY}L^{\prime}(x, y, f_{\lb}(x))(f_{\lb}(x)-f_{\mu}(x))\ d\P(x,y).
\eeqnal
The last equality holds true, because of the reproducing property (\ref{sec1.reproducingproperty}).

Since the loss function $L$ is convex with respect to the third argument, the following inequality is valid:
$$L^{\prime}(x, y, a)(b-a)\leq L(x, y, b)-L(x, y, a)\leq \Ls(x, y, b)-\Ls(x, y, a), \quad \forall a, b\in \R.$$
Let $a:=f_{\mu}(x)$ and $b:=f_{\lb}(x).$ We therefore obtain
$$L^{\prime}(x, y, f_{\mu}(x))(f_{\lb}(x)-f_{\mu}(x))\leq \Ls(x, y, f_{\lb}(x))-\Ls(x, y, f_{\mu}(x)).$$
Let $a:=f_{\lb}(x)$ and $b:=f_{\mu}(x).$ We therefore obtain
$$L^{\prime}(x, y, f_{\lb}(x))(f_{\mu}(x)-f_{\lb}(x))\leq \Ls(x, y, f_{\mu}(x))-\Ls(x, y, f_{\lb}(x)).$$

If we plug these two inequalities into (\ref{appendixb.difference}), we get
$$\hhnorm{f_{\lb}-f_{\mu}}\leq \Big(\frac{1}{2\lb}-\frac{1}{2\mu}\Big)\big(\Ex_{\P}\Ls(X, Y, f_{\mu}(X))-\Ex_{\P}\Ls(X, Y, f_{\lb}(X))\big).$$
The right hand side of above inequality is nonnegative, which implies that $\frac{1}{\lb}-\frac{1}{\mu}$ has the same sign as $\Ex_{\P}\Ls(X, Y, f_{\mu}(X))-\Ex_{\P}\Ls(X, Y, f_{\lb}(X)).$

If $\mu>\lb>0, $ then $\frac{1}{\lb}>\frac{1}{\mu}$ which in turn leads to $\Ex_{\P}\Ls(X, Y, f_{\mu}(X))-\Ex_{\P}\Ls(X, Y, f_{\lb}(X))\geq 0.$

The definition of $f_{\mu}$ tells us that
$$\Ex_{\P}\Ls(X, Y, f_{\mu}(X))+\mu\hhnorm{f_{\mu}}\leq\Ex_{\P}\Ls(X, Y, f_{\lb}(X))+\mu\hhnorm{f_{\lb}}. $$
It follows that
\beqna
0& \leq &\Ex_{\P}\Ls(X, Y, f_{\mu}(X))-\Ex_{\P}\Ls(X, Y, f_{\lb}(X))\\
&\leq &  \mu(\hhnorm{f_{\lb}}-\hhnorm{f_{\mu}})
= \mu(\hnorm{f_{\lb}}-\hnorm{f_{\mu}})(\hnorm{f_{\lb}}+\hnorm{f_{\mu}}),
\eeqna
from which we conclude that
$$0\leq \hnorm{f_{\lb}}-\hnorm{f_{\mu}}\leq \hnorm{f_{\lb}-f_{\mu}}.$$

Finally, we get that
$$\hnorm{f_{\lb}-f_{\mu}}\leq \frac {\mu}{2}\Big(\frac{1}{\lb}-\frac{1}{\mu}\Big)\big(\hnorm{f_{\lb}}+\hnorm{f_{\mu}}\big).$$

In the same way, we can prove for $\lb>\mu>0$ that
$$\hnorm{f_{\lb}-f_{\mu}}\leq \frac {\lb}{2}\Big(\frac{1}{\mu}-\frac{1}{\lb}\Big)\big(\hnorm{f_{\lb}}+\hnorm{f_{\mu}}\big).$$

We will now show the second assertion. Hence we assume the existence of a constant $r$ with $0<r<\min\{\lb, \mu\}$ and combine  $\hnorm{f_{\lb}}\leq \frac{1}{\lb}|L|_1\inorm{k}$ and $\hnorm{f_{\mu}}\leq \frac{1}{\mu}|L|_1\inorm{k}$ (see \citet[Proposition 3]{ ChristmannVanMessemSteinwart2009}). We obtain
\beqna
\hnorm{f_{\lb}-f_{\mu}}&\leq& \frac{1}{2}\Big(\frac{\max\{\lb, \mu\}}{\min\{\lb, \mu\}}-1\Big)\{\hnorm{f_{\lb}}+\hnorm{f_{\mu}}\}\\
&\leq & \frac{|L|_1\inorm{k}}{2}\Big(\frac{\max\{\lb, \mu\}}{\min\{\lb, \mu\}}-1\Big)\Big(\frac{1}{\lb}+\frac{1}{\mu}\Big)\\
&\leq & \frac{|L|_1\inorm{k}}{\min\{\lb^2, \mu^2\}}|\lb-\mu|\leq  \frac{|L|_1\inorm{k}}{r^2}\cdot |\lb-\mu|= \mathcal{O}( |\lb-\mu| ),
\eeqna
which yields the second assertion of the theorem.
\qedr
\end{proofof}

To prove Theorem \ref{sec1.thm2}, we will use the triangle inequality to obtain the following error decomposition
\be\label{appendixb.decomp1}
\inorm{\fsaaa-\fsbbb}\leq \inorm{\fsaaa-\fsbaa}+\inorm{\fsbaa-\fsbba}+\inorm{\fsbba-\fsbbb}.
\ee

The following lemma gives an upper bound for the third norm on the right hand side of (\ref{appendixb.decomp1}).
\begin{lemma}\label{appendixb.lem1}
If Assumptions \ref{sec2.assumption-spaces1}, \ref{sec2.assumption-kernel1}, \ref{sec1.assumption2}, and \ref{sec1.loss} are valid, then, for all $\lb_2>\frac{1}{2}|L^\prime|_1\kappa^2,$
$$\inorm{\fsbba-\fsbbb}\leq \frac{|L|_1}{2\lb_2-|L^\prime|_1\kappa^2}\sup_{x\in\cX}\inorm{\kb{x}-\ka{x}}.$$
\end{lemma}

\begin{proof}
To shorten the notations in the proof, we define $\lb:=\lb_2,\ \P:=\P_2,$ $f_{1}:=\fsbba$ and $f_{2}:=\fbbb.$
By the representer theorem for SVMs, see Theorem \ref{sec1.representerthm} , we know that $\inorm{f_{1}-f_{2}}$
can be bounded as follows:
\beqna
&&\inorm{f_{1}-f_{2}}\\
&=&\frac{1}{2\lb}\Big\|\int_{\cX\times\cY}L^{\prime}(x, y, f_{2}(x))\Phi_2(x)\ d\P(x,y)-\int_{\cX\times\cY}L^{\prime}(x, y,f_{1}(x))\Phi_1(x)\ d\P(x,y)\Big\|_\infty\\
&=& \frac{1}{2\lb}\Big\|\int_{\cX\times\cY}\bigl(L^{\prime}(x, y,f_{2}(x))\Phi_2(x)-L^{\prime}(x, y,f_{1}(x))\Phi_1(x)\bigr)\ d\P(x,y)\Big\|_\infty\\
&\leq&  \frac{1}{2\lb}\Big\|\int_{\cX\times\cY}\bigl(L^{\prime}(x, y, f_{2}(x))\Phi_2(x)-L^{\prime}(x, y, f_{1}(x))\Phi_2(x)\bigr)\ d\P(x,y)\Big\|_\infty\\
& &+ \frac{1}{2\lb}\Big\|\int_{\cX\times\cY}\bigl(L^{\prime}(x, y, f_{1}(x))\Phi_2(x)-L^{\prime}(x, y, f_{1}(x))\Phi_1(x)\bigr)\ d\P(x,y)\Big\|_\infty\\
&\leq &  \frac{|L^\prime|_1}{2\lb}\inorm{f_{2}-f_{1}}\sup_{x\in \cX}\inorm{\Phi_2(x)}
+ \frac{1}{2\lb}\inorm{L^{\prime}(x, y,f_{1}(x))}\sup_{x\in\cX}\inorm{\Phi_2(x)-\Phi_1(x)}\\
&\leq &  \frac{|L^\prime|_1\inorm{k_2}^2}{2\lb}\inorm{f_{2}- f_{1}}
+ \frac{|L|_1}{2\lb}\sup_{x\in\cX}\inorm{\Phi_2(x)-\Phi_1(x)}.
\eeqna
The second inequality comes from the Lipschitz continuity of $L^\prime.$

If $\lb>\frac{1}{2}|L^\prime|_1\inorm{k_2}^2,$ we can solve $\inorm{f_1-f_{2}}$ from above inequality and get that
$$\inorm{f_{1}-f_{2}}\leq \frac{|L|_1}{2\lb-|L^\prime|_1\inorm{k_2}^2}\sup_{x\in\cX}\inorm{\Phi_2(x)-\Phi_1(x)}.$$

Obviously, the kernels $k_1$ and $k_2$ can change their roles when we do the decomposition on the second inequality. Hence we get an analogous inequality for case $\lb>\frac{1}{2}|L^\prime|_1\inorm{k_1}^2.$ This gives the assertion.
\end{proof}

\begin{proofof}{\textbf{Proof of Theorem \ref{sec1.thm2}}}
	We first prove the first term of (\ref{appendixb.decomp1}).
The reproducing property (\ref{sec1.reproducingproperty}) yields that
$$ \inorm{\fsaaa-\fsbaa}\leq \inorm{k_1} \hanorm{\fsaaa-\fsbaa}.$$
Hence it suffices to bound $\hanorm{\fsaaa-\fsbaa}.$

The representer theorem for SVMs, see Theorem \ref{sec1.representerthm}, and the properties of Bochner integrals, see e.g. \citet[Chap. 3.10, p. 364 ff]{DenkowskiEtAl2003} and Lemma \ref{appendixc.lem1}, tell us that
\beqna
& &\hanorm{\fsaaa-\fsbaa} \\
&\leq & \frac{1}{\lb_1}\Big\| \int_{\cX\times\cY} L^{\prime}(x, y, f_{\P_1, \lb_1, k_1} (x))\Phi_1(x)\ d\P_1(x, y)- \int_{\cX\times\cY} L^{\prime}(x, y, f_{\P_1, \lb_1, k_1} (x))\Phi_1(x)\ d\P_2(x, y)
\Big\|_{H_1}\\
& \leq & \frac{1}{\lb_1}\int_{\cXY}\hanorm{L^{\prime}(x, y, f_{\P_1, \lb_1, k_1} (x))\Phi_1(x)}\ d|\P_1-\P_2|(x, y)\\
&\leq &  \frac{1}{\lb_1}\int_{\cXY}\sup_{(x, y)\in \cXY}|L^{\prime}(x, y, f_{\P_1, \lb_1, k_1} (x))|\cdot \sup_{x\in\cX}\hanorm{\Phi_1(x)}\ d|\P_1-\P_2|(x, y)\\
&\leq & \frac{1}{\lb_1}\sup_{(x, y)\in \cXY}|L^{\prime}(x, y, f_{\P_1, \lb_1, k_1} (x))|\cdot \sup_{x\in\cX}\hanorm{\Phi_1(x)}\cdot \tvnorm{\P_1-\P_2}\\
&\leq & \frac{1}{\lb_1}\inorm{k_1}|L|_1\tvnorm{\P_1-\P_2}.
\eeqna
This yields the assertion.

An application of the above assertion, Theorem \ref{sec1.thm1}, and Lemma \ref{appendixb.lem1} yields that
\beqna
& &\inorm{\fsaaa-\fsbbb}\\
&\leq & \frac{\inorm{k_1}^2|L|_1}{\lb_1}\cdot\tvnorm{\P_1-\P_2}+\frac{\inorm{k_1}^2|L|_1}{\min\{\lb_1^2, \lb_2^2\}}\cdot |\lb_1-\lb_2| +\frac{|L|_1}{2\lb_2-\kappa^2|L^\prime|_1}\cdot\sup_{ x\in\cX}\inorm{\kb{x}-\ka{x}}.
\eeqna


If we split $\inorm{\fsaaa-\fsbbb}$ into three parts as stated in (\ref{appendixb.decomp1}), there are six different decompositions. If we take all six cases into account and assume $\min\{\lb_1, \lb_2\}>r:=\frac{1}{2}\kappa^2|L^\prime|_1,$ we get the assertion (\ref{sec1.supnorm}).
\qedr
\end{proofof}
\begin{proofof}{\textbf{Proof of Corollary \ref{sec1.cor1}}}
	The inequality (\ref{sec1.supnorm-func}) follows immediately from the assumption that the positive constant $s$ is smaller than $\min\{\lb_1,\lb_2\} - r.$

Hence we only have to show the validity of (\ref{sec1.supnorm-risk}). To shorten the notation in the proof, we define $f_1:=\fsbbb$ and $f_2:=\fsbbb.$
The definitions of $\Lsriska{f_1}$ and $\Lsriskb{f_2}$ yield that
\beqna
& &\Lsriska{f_1}-\Lsriskb{f_2}\\
&=& \int_{\cXY}\big(L(x, y, f_1(x))-L(x, y, 0)\big)\ d\P_1(x, y)-\int_{\cXY}\big(L(x, y, f_2(x))-L(x, y, 0)\big)\ d\P_2(x, y).
\eeqna
Plugging $\mp\int_{\cXY}\big(L(x, y, f_2(x))-L(x, y, 0)\big)\ d\P_1(x, y) $ into the above equation and further noticing that $L$ is a Lipschitz continuous loss function, we get from Lemma \ref{appendixc.lem1} that
\beqnal
& &\big|\Lsriska{f_1}-\Lsriskb{f_2}\big|\nonumber\\
& \leq &  \int_{\cXY}\big|L(x, y, f_1(x))-L(x, y, f_2(x))\big|\ d\P_1(x, y)\nonumber\\
& &+\int_{\cXY}\big|L(x, y, f_2(x))-L(x, y, 0)\big|\ d|\P_1-\P_2|(x, y)\nonumber\\
&\leq& \int_{\cXY}|L|_{1}|f_1(x)-f_2(x)|\ d\P_1(x, y)+\int_{\cXY}|L|_{1}|f_2(x)|\  d|\P_1-\P_2|(x, y)\\
&\leq & |L|_{1}\inorm{f_1-f_2}+ |L|_{1}\inorm{f_2}\tvnorm{\P_1-\P_2}.\nonumber
\eeqnal

Since $\hbnorm{f_2}\leq \frac{1}{\lb_2}|L|_1\inorm{k_2}$ (see \citet[Proposition 3]{ ChristmannVanMessemSteinwart2009}) and $\lb_2>\frac{1}{2}\kappa^2|L'|_1$, we obtain
$$\inorm{f_2}\leq  \frac{1}{\lb_2}|L|_1\inorm{k_2}^2\leq \frac{1}{\lb_2}|L|_1\kappa^2\leq \frac{2|L|_1}{|L'|_1}.$$
The above upper bound is a constant independent of $\P_1, \P_2, \lb_1, \lb_2, k_1$, and $k_2.$

Therefore, if we combine this result with Theorem \ref{sec1.thm2},  we obtain
\beqna
& &\big|\Lsriska{f_1}-\Lsriskb{f_2}\big|\\
&\leq & |L|_{1}\Big(c_1(L)\cdot\tvnorm{\P_1-\P_2}+c_2(L)\cdot |\lb_1-\lb_2|+ \bar{c}_3(L)\cdot\sup_{ x\in\cX}\inorm{\kb{x}-\ka{x}}\Big)\\
& &+\frac{2|L|_1^2}{|L'|_1}\tvnorm{\P_1-\P_2}.
\eeqna

The desired inequality (\ref{sec1.supnorm-risk}) follows by rearranging the factors to compute the constants.
	\qedr

\end{proofof}
Now we are in a position to bound the $H_1$-norm of the difference $\fsaaa-\fsbbb$ for the case that $H_2\subseteq H_1.$ We use the triangle inequality to obtain the following decomposition:
\be\label{appendixb.decomphnorm}
\hanorm{\fsaaa-\fsbbb}\leq \hanorm{\fsaaa-\fsbaa}+\hanorm{\fsbaa-\fsbba}+\hanorm{\fsbba-\fsbbb}.
\ee
To bound (\ref{appendixb.decomphnorm}), we first bound the third norm for the case of a differentiable loss.

\begin{lemma}\label{appendixb.diffker}
		Let Assumptions \ref{sec2.assumption-spaces1}, \ref{sec2.assumption-kernel1}, and \ref{sec1.assumption2} be satisfied. Assume that $H_1$ and $H_2$ satisfy (\ref{sec1.inclusion}). If $L$ is additionally a differentiable loss function, then, for all $\lb_2>0,$ we have
			$$\hanorm{\fsbba-\fsbbb}\leq \frac{|L|_1}{2\lb_2}\sup_{x\in \cX}\hanorm{\ka{x}-\kb{x}}.$$
	\end{lemma}

\begin{proof}
To shorten the notation, we will use the following abbreviations in this proof: $\lb:=\lb_2,$ $\P:=\P_2,$ $f_{1}:=\fsbba,$ and $f_{2}:=\fsbbb.$
	Since $H_2\subseteq H_1,$ we have that $ f_{2}\in H_1$ and $\hanorm{f_{1}- f_{2}}$ is well-defined. The representer theorem, see Theorem \ref{sec1.representerthm}, tells us that, for all $\P\in\PM(\cXY),$
	\beqna
	\hhanorm{f_{1}-f_{2}}&=& \langle f_{1}-f_{2},\ f_{1}-f_{2} \rangle_{H_1}\\
	&=&  \Big\langle f_{1}-f_{2},\  -\frac{1}{2\lb}\int_{\cX\times\cY}L^{\prime}(x, y, f_{1}(x))\Phi_1(x)\ d\P(x,y)\\
	& &~~+\frac{1}{2\lb}\int_{\cX\times\cY}L^{\prime}(x, y,f_{2}(x))\Phi_2(x)\ d\P(x,y)\Big\rangle_{H_1}.
	\eeqna
	Plugging a zero term into the last term of the above inner product, we know that
	\beqna
	& & -\frac{1}{2\lb}\int_{\cX\times\cY}L^{\prime}(x, y, f_{1}(x))\Phi_1(x)\ d\P(x,y)
	+\frac{1}{2\lb}\int_{\cX\times\cY}L^{\prime}(x, y, f_{2}(x))\Phi_2(x)\ d\P(x,y)\\
	&=&  -\frac{1}{2\lb}\int_{\cX\times\cY}L^{\prime}(x, y, f_{1}(x))\Phi_1(x)\ d\P(x,y)
	+\frac{1}{2\lb}\int_{\cX\times\cY}L^{\prime}(x, y, f_{2}(x))\Phi_1(x)\ d\P(x,y)\\
	& & - \frac{1}{2\lb}\int_{\cX\times\cY}L^{\prime}(x, y, f_{2}(x))\Phi_1(x)d\P(x,y)+\frac{1}{2\lb}\int_{\cX\times\cY}L^{\prime}(x, y, f_{2}(x))\Phi_2(x)\ d\P(x,y).
	\eeqna
	Therefore, the reproducing property (\ref{sec1.reproducingproperty}) yields
	\beqna
\hhanorm{f_{1}- f_{2}}
	&=& -\frac{1}{2\lb}\int_{\cX\times\cY}(L^{\prime}(x, y, f_{1}(x))-L^{\prime}(x, y, f_{2}(x)))(f_{1}(x)-f_{2}(x))\ d\P(x,y)\\
	& &- \frac{1}{2\lb}\Big\langle \int_{\cX\times\cY}L^{\prime}(x, y, f_{2}(x))(\Phi_1(x)-\Phi_2(x))\ d\P(x,y),\ f_{1}- f_{2}\Big\rangle_{H_1}.
	\eeqna
	Since $L$ is convex and differentiable with respect to its third argument, and hence $L^\prime(x, y, \cdot)$ is non-decreasing, for any choice of $(x, y)\in\cX\times\cY,$ we obtain for the first integral on the right hand side of the above equation that
	
	$$\int_{\cX\times\cY}\big(L^{\prime}(x, y, f_{1}(x))-L^{\prime}(x, y, f_{2}(x))\big)( f_{1}(x)-f_{2}(x))\ d\P(x,y)\geq 0.$$
	It follows that
	\beqna
	\hhanorm{ f_{1}-f_{2}}
	&\leq& -\frac{1}{2\lb_2}\Big\langle \int_{\cX\times\cY}L^{\prime}(x, y, f_{2}(x))(\Phi_1(x)-\Phi_2(x))\ d\P(x,y),\  f_{1}-f_{2}\Big\rangle_{H_1}\\
	&\leq &\frac{|L|_1}{2\lb}\hanorm{ f_{1}-f_{2}}\sup_{x\in \cX}\hanorm{\Phi_1(x)-\Phi_2(x)}.
	\eeqna
	Thus we have, for all $\lb>0,$
	$$\hanorm{ f_{1}-f_{2}}\leq \frac{|L|_1}{2\lb}\sup_{x\in \cX}\hanorm{\Phi_1(x)-\Phi_2(x)}.$$
\end{proof}

\begin{lemma}\label{appendixb.diffloss}
	Let Assumptions \ref{sec2.assumption-spaces1}, \ref{sec2.assumption-kernel1}, and \ref{sec1.assumption2} be satisfied. Assume that $H_1$ and $H_2$ satisfy (\ref{sec1.inclusion}). If $L$ is additionally a differentiable loss function, then, for all $\lb_1,\lb_2>0,$ we have
	\beqnal\label{sec1.hnorm1}
	& &\hanorm{\faaa-\fbbb}\nonumber\\
	&\leq & c_1^\prime(L, \lb_1, \lb_2)\tvnorm{\P_1-\P_2}+c_2^\prime(L, \lb_1, \lb_2)|\lb_1 -\lb_2|+c_3^\prime(L, \lb_1, \lb_2)\sup_{x\in \cX}\hanorm{\ka{x}-\kb{x}},\nonumber
	\eeqnal
	where $c_1^\prime(L, \lb_1, \lb_2):=\frac{\kappa|L|_1}{\min\{\lb_1, \lb_2\}}, c_2^\prime(L, \lb_1, \lb_2):=\frac{\kappa|L|_1}{\min\{\lb_1^2, \lb_2^2\}},$ and $c_3^\prime(L, \lb_1, \lb_2):=\frac{|L|_1}{2\min\{\lb_1, \lb_2\}}.$	
	\end{lemma}

\begin{proof}
Of course, $H_2\subseteq H_1$ implies $\fsbbb\in H_1$ and therefore $\hanorm{\fsaaa-\fsbbb}$ is well-defined.

The representer theorem for SVMs, see Theorem \ref{sec1.representerthm}, and the properties of Bochner integrals tell us that
\beqna
& &\hanorm{\fsaaa-\fsbaa} \\
&\leq & \frac{1}{\lb_1}\Big\| \int_{\cX\times\cY} L^{\prime}(x, y, f_{\P_1, \lb_1, k_1} (x))\Phi_1(x)\ d\P_1(x, y)- \int_{\cX\times\cY} L^{\prime}(x, y, f_{\P_1, \lb_1, k_1} (x))\Phi_1(x)\ d\P_2(x, y)
\Big\|_{H_1}\\
& \leq & \frac{1}{\lb_1}\hnorm{L^{\prime}(x, y, f_{\P_1, \lb_1, k_1} (x))\Phi_1(x)}\tvnorm{\P_1-\P_2}\leq  \frac{1}{\lb_1}\inorm{k_1}|L|_1\tvnorm{\P_1-\P_2}.
\eeqna
This gives an upper bound to the first norm on the right hand side of equation (\ref{appendixb.decomphnorm}).

An application of the above assertion, Theorem \ref{sec1.thm1} and Lemma \ref{appendixb.diffker} yields that
	\beqnal\label{sec1.hnorm1}
& &\hanorm{\faaa-\fbbb}\nonumber\\
&\leq & \frac{\inorm{k_1}|L|_1}{\lb_1}\tvnorm{\P_1-\P_2}+\frac{\inorm{k_1}|L|_1}{\min\{\lb_1^2, \lb_2^2\}}|\lb_1 -\lb_2|+\frac{|L|_1}{2\lb_2}\sup_{x\in \cX}\hanorm{\ka{x}-\kb{x}}.\nonumber
\eeqnal

When we apply the triangle inequality to the error decomposition (\ref{appendixb.decomphnorm}), there are six different decompositions. If we take all six cases into account, we get our desired result.
\end{proof}
\begin{proofof}{\textbf{Proof of Theorem \ref{sec1.thm3}}}
(i) Lemma \ref{appendixb.diffloss} gives the first assertion.

(ii) We only need to prove the assertion holds true for the case of a non-differentiable loss, where the loss function $L(x, y, t)$  can be represented by a margin-based loss function $\tilde{L}(yt)$ for classification or by a distance-based loss function $\tilde{L}(y-t)$ for regression with Lipschitz constant $|\tilde{L}|_1$. Let $\d_{j}\in (0, 1).$ We use the standard technique of convolution, see e. g. \citet[p.148]{CheneyLight2000}, to define a convex and differentiable function $\tilde{L}^\star_{\d}$ on $\Rd$ by
$$\tilde{L}^\star_{\d}(\xi)=\int_{0}^{1}\tilde{L}^\star(\xi-\d\theta)d\theta=\frac{1}{\d}\int_{\xi-\d}^{\xi}\tilde{L}^\star(u)du$$
to approximate the shifted loss function $\tilde{L}^\star.$

It is easy to check that $\tilde{L}^\star_{\d}$ is convex, differentiable, and Lipschitz continuous with the same Lipschitz constant $|\tilde{L}|_1.$ The approximation is valid, because, for every $\xi\in\R,$
$$|\tilde{L}^\star_{\d}(\xi)-\tilde{L}^\star(\xi)|=\Big|\int_{0}^{1}(\tilde{L}^\star(\xi-\d\theta)-\tilde{L}^\star(\xi))d\theta\Big|\leq \int_{0}^{1}|\tilde{L}|_1\d\theta d\theta\leq \frac{|\tilde{L}|_1}{2}\d.$$

Hence
\be\label{appendixb.uniformconvergence}
\inorm{\tilde{L}^\star_{\d}-\tilde{L}^\star}=\mathcal{O}(\d), \ \hbox{as } \d\to 0_+.
\ee

An SVM associated with $\tilde{L}^\star_{\d}$ can be defined as
$$f_{\P, \lb, k, (\d)}=\arg\min_{f\in H}\Big(\Ex_{\P}\tilde{L}^\star_{\d}(X, Y, f(X))+\lb\hhnorm{f}\Big).$$

We now show the weak convergence in $H$ of $f_{\P, \lb, k, (\d_j)}$ to $f_{\P, \lb, k}:=\arg\min_{f\in H}\{\Ex_{\P}\tilde{L}^\star(X, Y, f(X))+\lb\hhnorm{f}\},$ for $(\d_j)_{j\in \N}$ with $\d_j\to 0$ and $\d_j\in (0, 1), j\in\N.$

\citet[Proposition 3]{ChristmannVanMessemSteinwart2009} showed
\be\label{appendixb.hbound}
\hnorm{f_{\P, \lb, k, (\d)}}\leq \lb^{-1}|\tilde{L}|_1\inorm{k}.
\ee
Any closed ball $B_R=\{f\in H, \hnorm{f}\leq R\}$ of the Hilbert space $H$ with a finite radius $R>0$ is weakly compact. Hence there exists a decreasing sequence $(\d_j)_{j\in \N},$ where all $\d_j\in (0, 1),$ such that $\lim_{j\to\infty}\d_{j}=0$ and $f_{\P, \lb, k, (\d_j)}$ weakly converges to some function $g_{\P, \lb, k}\in H.$ That is
\be\label{appendixb.weaklyconvergence}
\lim_{j\to\infty}\langle f_{\P, \lb, k, (\d_j)}, f\rangle_{H}=\langle g_{\P, \lb, k}, f\rangle_{H},\ \forall f\in H.
\ee
Let us consider two special cases of (\ref{appendixb.weaklyconvergence}).

If $f=g_{\P, \lb, k}$ in (\ref{appendixb.weaklyconvergence}), then we obtain by the Cauchy-Schwartz inequality that
$$\hhnorm{g_{\P, \lb, k}}=\lim_{j\to\infty}\langle f_{\P, \lb, k, (\d_j)}, g_{\P, \lb, k}\rangle_{H}\leq \hnorm{g_{\P, \lb, k}}\liminf_{j\to\infty}\hnorm{f_{\P, \lb, k, (\d_j)}}.$$
Therefore, we get by (\ref{appendixb.hbound}) that
\be\label{appendixb.upbound}
\hnorm{g_{\P, \lb, k}}\leq \liminf_{j\to\infty}\hnorm{f_{\P, \lb, k, (\d_j)}}\leq \lb^{-1}|\tilde{L}|_1\inorm{k}.
\ee
Now we consider the special case of $f=\kk{x}$ in (\ref{appendixb.weaklyconvergence}). The reproducing property (\ref{sec1.reproducingproperty}) yields
\be\label{appendixb.pointconvergence}
g_{\P, \lb, k}(x)=\langle g_{\P, \lb, k}, \kk{x}\rangle_{H}=\lim_{j\to
\infty}\langle f_{\P, \lb, k, (\d_j)}, \kk{x}\rangle_{H}=\lim_{j\to\infty}f_{\P, \lb, k, (\d_j)}(x).
\ee
The Lebesgue Dominated Theorem gives
$$\Ex_{\P}[\tilde{L}^\star(X, Y, g_{\P, \lb, k}(X)]=\lim_{j\to\infty}\Ex_{\P}[\tilde{L}^\star(X, Y, f_{\P, \lb, k, (\d_j)}(X))].$$
The uniform estimate (\ref{appendixb.uniformconvergence}) in connection with (\ref{appendixb.pointconvergence}) yields
\beqnal\label{appendixb.expectation}
\lim_{j\to\infty}\Ex_{\P}[\tilde{L}^\star_{\d_j}(X, Y, f_{\P, \lb, k, (\d_j)}(X))]&=&\lim_{j\to\infty}\int_{\cX\times\cY}\tilde{L}^\star_{\d_j}(x, y, f_{\P, \lb, k, (\d_j)}(x))\ d\P(x, y)\nonumber\\
&=&\lim_{j\to\infty}\int_{\cX\times\cY}\tilde{L}^\star(x, y,  f_{\P, \lb, k, (\d_j)}(x))\ d\P(x, y)\nonumber\\
&=& \int_{\cX\times\cY}\tilde{L}^\star(x, y, g_{\P, \lb, k}(x))\ d\P(x, y)\nonumber\\
&=& \Ex_{\P}[\tilde{L}^\star(X, Y, g_{\P, \lb, k}(X))].
\eeqnal

Combining (\ref{appendixb.upbound}) and (\ref{appendixb.expectation}), we obtain
\be\label{appendixb.regularizedrisk}
\Ex_{\P}[\tilde{L}^\star(X, Y, g_{\P, \lb, k}(X))]+\lb\hhnorm{g_{\P, \lb, k}}\leq\liminf_{j\to\infty}\Big(\Ex_{\P}[\tilde{L}^\star_{\d_j}(X, Y, f_{\P, \lb, k, (\d_j)}(X))]+\lb\hhnorm{f_{\P, \lb, k, (\d_j)}}\Big).
\ee
By the definition of $f_{\P, \lb, k, (\d_j)},$ we know
\beqnal\label{appendixb.min}
& &\liminf_{j\to\infty}\Big(\Ex_{\P}[\tilde{L}^\star_{\d_j}(X, Y, f_{\P, \lb, k, (\d_j)}(X))]+\lb\hhnorm{f_{\P, \lb, k, (\d_j)}}\Big)\\
&\leq & \liminf_{j\to\infty}\Big(\Ex_{\P}[\tilde{L}^\star_{\d_j}(X, Y, f_{\P, \lb, k}(X))]+\lb\hhnorm{f_{\P, \lb, k}}\Big)\nonumber\\
&\leq& \Ex_{\P}[\tilde{L}^\star(X, Y, f_{\P, \lb, k}(X))]+\lb\hhnorm{f_{\P, \lb, k}}.
\eeqnal
Hence (\ref{appendixb.regularizedrisk}) and (\ref{appendixb.min}) lead to
$$\Ex_{\P}[\tilde{L}^\star(X, Y, g_{\P, \lb, k}(X))]+\lb\hhnorm{g_{\P, \lb, k}}\leq\Ex_{\P}[\tilde{L}^\star(X, Y, f_{\P, \lb, k}(X))]+\lb\hhnorm{f_{\P, \lb, k}} .$$

The strict convexity of the regularized functional $f\mapsto\Ex_{\P}[\tilde{L}^\star(X, Y, f(X))]+\lb\hhnorm{f}$ on $H$ guarantees the uniqueness of the minimizer, which implies the identity $g_{\P, \lb, k}=f_{\P, \lb, k}$ and the weak convergence
\be\label{appendixb.weak}
\lim_{j\to\infty}\langle f_{\P, \lb, k, (\d_j)}, f\rangle_{H}=\langle f_{\P, \lb, k}, f\rangle_{H},\ \forall f\in H.
\ee

 In the rest of the proof, we focus on estimating $\hanorm{\fsaaa-\fsbbb} .$ The triangle inequality yields
\beqna
& &\hanorm{\fsaaa-\fsbbb}\\
&\leq& \hanorm{\fsaaa-f_{\P_1, \lb_1, k_1, (\d_j)}}+\hanorm{f_{\P_1, \lb_1, k_1, (\d_j)}-f_{\P_2, \lb_2, k_2, (\d_j)}}+\hanorm{f_{\P_2, \lb_2, k_2, (\d_j)}-\fsbbb}.
\eeqna
Now apply (\ref{appendixb.weak}) to $\P=\P_1, \lb=\lb_1, k=k_1$ and $H=H_1.$ Then we get
$$\lim_{j\to\infty}\langle f_{\P_1, \lb_1, k_1, (\d_j)}, \fsaaa\rangle_{H_1}=\hhanorm{\fsaaa}$$
and
$$\lim_{j\to\infty}\hhanorm{f_{\P_1, \lb_1, k_1, (\d_j)}}=\hhanorm{\fsaaa},$$
which implies that
\beqna
&&\lim_{j\to\infty}\hhanorm{\fsaaa-f_{\P_1, \lb_1, k_1, (\d_j)}}\\
&=&\hhanorm{\fsaaa}+\lim_{j\to\infty}\hhanorm{f_{\P_1, \lb_1, k_1, (\d_j)}}-2\lim_{j\to\infty}\langle f_{\P_1, \lb_1, k_1, (\d_j)}, \fsaaa\rangle_{H_1}=0.
\eeqna

In the same way, we can prove
$$\lim_{j\to\infty}\hanorm{f_{\P_2, \lb_2, k_2, (\d_j)}-\fsbbb}=0.$$

We know that $\tilde{L}^\star_{\d_j}$ is a convex, differentiable and Lipschitz continuous loss function with constant $|\tilde{L}|_1.$ Hence Lemma \ref{appendixb.diffloss} tells us that, for all $\lb_1, \lb_2>0,$
\beqnal\label{sec1.hnorm1}
	& &\hanorm{f_{\P_1, \lb_1, k_1, (\d_j)}-f_{\P_2, \lb_2, k_2, (\d_j)}}\nonumber\\
	&\leq & \tilde{c}_1^\prime(\tilde{L}, \lb_1, \lb_2)\tvnorm{\P_1-\P_2}+\tilde{c}_2^\prime(\tilde{L}, \lb_1, \lb_2)|\lb_1 -\lb_2|+\tilde{c}_3^\prime(\tilde{L}, \lb_1, \lb_2)\sup_{x\in \cX}\hanorm{\ka{x}-\kb{x}},\nonumber
	\eeqnal
	where $\tilde{c}_1^\prime(\tilde{L}, \lb_1, \lb_2):=\frac{\kappa|\tilde{L}|_1}{\min\{\lb_1, \lb_2\}}, \tilde{c}_2^\prime(\tilde{L}, \lb_1, \lb_2):=\frac{\kappa|\tilde{L}|_1}{\min\{\lb_1^2, \lb_2^2\}},$ and $\tilde{c}_3^\prime(\tilde{L}, \lb_1, \lb_2):=\frac{|\tilde{L}|_1}{2\min\{\lb_1, \lb_2\}}.$

Therefore, this yields the assertion.
\qedr
\end{proofof}

\subsection{Appendix C: Proofs for results in Section \ref{sec2}}\label{appendixc}
   The proof of Theorem \ref{sec2.thm1} is based on the following
error decomposition:
\beq
 & & \inorm{\faaa - \fbbb} \nonumber \\
  &   \le   &
  \inorm{\faaa - \fbaa} + \inorm{\fbaa - \fbba} + \inorm{\fbba - \fbbb}\,.  \label{appendixc.f1}
\eeq
The following lemmas give upper bounds for the three norms on the right hand side of {(\ref{appendixc.f1})}. The proof of Theorem \ref{sec2.thm1}
will then follow by combining these upper bounds.

A major tool to prove these lemmas is the following representer theorem, see \citet[Thm. 4.3]{ChristmannZhou2016a}. Note that we specialized their result to the case of a \emph{convex} pairwise loss function, because we need in particular the inequality {(\ref{thm.representer.f4})}. There are of course pairwise learning algorithms involving non-convex loss functions, see \citet{HuFanWuZhou2015} and \citet{FanHuWuZhou2016}. Please note that
the expectations in the next theorem are Bochner integrals. We refer to
\citet[Chapter 3.10]{DenkowskiEtAl2003} for details on Bochner integrals.

\begin{theorem}[\textbf{Representer theorem for pairwise learning}\label{sec2.representerthm}]
Let Assumptions \ref{sec2.assumption-spaces1}, \ref{sec2.assumption-kernel1}, and \ref{sec2.assumption-loss1} be valid.
Then we have, for all $\P,\Q\in\PM(\cXY)$ and all $\lb\in(0,\infty)$:
\begin{enumerate}
\item The estimator $\fPLs$ defined as the minimizer of
$\min_{f\in H} \big( \RP{\Ls}{f}+\lb\hhnorm{f}\bigr)$ exists, is unique, and satisfies
\begin{eqnarray} \label{thm.representer.f1}
 \fPLs  =  - \frac{1}{2\lb} \Ex_{\P^2}
                       \big[
                          h_{5,\P}(\XYXY) \Phi(X) +  h_{6,\P}(\XYXY) \Phi(\tiX)
                       \big],
\end{eqnarray}
where $h_{5,\P}$ and $h_{6,\P}$ denote the partial derivatives
\begin{eqnarray}
   h_{5,\P}(\XYXY)   :=   \DfiveL{\fPLs} ~ \label{thm.representer.f2}\\
   h_{6,\P}(\XYXY)   := \DsixL{\fPLs}. \label{thm.representer.f3}
\end{eqnarray}
\item Furthermore,
\begin{eqnarray}
 & & \hnorm{\fPLs -\fQLs} \label{thm.representer.f4}\\
 & \le & \frac{1}{\lb} \Big\|
                      \Ex_{\P^2}
                       \big[
                         h_{5,\P}(\XYXY) \Phi(X)  + h_{6,\P}(\XYXY) \Phi(\tiX)
                       \big]  \nonumber \\
           & & ~~~- \Ex_{\Q^2}
                       \big[
                         h_{5,\P}(\XYXY) \Phi(X)  + h_{6,\P}(\XYXY) \Phi(\tiX)
                       \big]  \Big\|_H \, . \nonumber
\end{eqnarray}
\end{enumerate}
\end{theorem}


\begin{lemma}\label{appendixc.lem3}
If Assumptions \ref{sec2.assumption-spaces1},
\ref{sec2.assumption-kernel1}, and
\ref{sec2.assumption-loss1} are satisfied, then
\beq \label{appendixc.lem3f1}
   \inorm{\fbaa - \fbba}
   \le
   \frac{\inorm{k_1}}{2} \, \Bigl(\frac{\max\{\lb_1,\lb_2\}}{\min\{\lb_1,\lb_2\}} - 1\Bigr) \, \bigl(\hanorm{\fbaa} + \hanorm{\fbba}\bigr).
\eeq
If there exists a constant $r\in (0,\infty)$ such that
$\min\{\lb_1,\lb_2\}>r$, then
\beq
   \inorm{\fbaa - \fbba}
   \le
   \frac{|L|_1 \inorm{k_1}^2}{r^2} \cdot |\lb_1-\lb_2|
   = \mathcal{O}( |\lb_1-\lb_2| ) \,.
\eeq
\end{lemma}

\begin{proof}
The assertion of the lemma is obviously valid, if $\lb_1=\lb_2$.

From now on, we will assume w.l.o.g. that $\lb_1 > \lb_2$.
To shorten the notation in this proof, we define $\P:=\P_2$,
$\lb:=\lb_1$, $\mu:=\lb_2$, $k:=k_1$, and $H:=H_1$.
Furthermore, we write
$$
  \flb := \fbaa \quad {\mathrm{and}} \quad \fmu:=\fbba.
$$
Hence, we have to show that
\be \label{appendixc.lem2f2}
  \inorm{\flb - \fmu} \le \frac{\inorm{k}}{2} \, \Bigl(\frac{\lb}{\mu} - 1\Bigr) \, \bigl(\hnorm{\flb} + \hnorm{\fmu}\bigr). \nonumber
\ee
To shorten the notation we write the partial derivatives of $\Ls$ with respect to the
$i^{th}$ argument by
\be
   D_i \Ls \circ f(x,y,\tix,\tiy) := D_i \Ls(x,y,\tix, \tiy, f(x), f(\tix)),
   \quad i\in\{5,6\}. \nonumber
\ee
We use again the representer theorem for pairwise loss functions, see
Theorem \ref{sec2.representerthm},
\beq
  \flb - \fmu
  & = & -\frac{1}{2\lb} \int \bigl[ \Dfive{\flb}  \Phi(x) + \Dsix{\flb} \Phi(\tix)\bigr] \,d\P^2(\xyxy) \nonumber\\
  &   & + \frac{1}{2\mu} \int \bigl[ \Dfive{\fmu}  \Phi(x) + \Dsix{\fmu} \Phi(\tix) \bigr] \,d\P^2(\xyxy) \,. \nonumber
\eeq
Hence
\beq
 &   & \hhnorm{\flb -\fmu}  \nonumber \\
 & = & \left\langle \flb -\fmu, \, \flb -\fmu \right\rangle_H  \nonumber  \\
 & = & \frac{1}{2\mu} \left\langle \int \bigl[ \Dfive{\fmu}  \Phi(x) + \Dsix{\fmu} \Phi(\tix) \bigr] \,d\P^2(\xyxy), \, \flb -\fmu \right\rangle_H \nonumber \\
 &   & -\frac{1}{2\lb} \left\langle \int \bigl[ \Dfive{\flb}  \Phi(x) + \Dsix{\flb} \Phi(\tix) \bigr] \,d\P^2(\xyxy), \, \flb -\fmu \right\rangle_H \nonumber \\
 & = & \frac{1}{2\mu} \int \bigl[ \Dfive{\fmu} (\flb(x)-\fmu(x)) +
                                  \Dsix{\fmu} (\flb(\tix)-\fmu(\tix)) \bigr]
                                  \,d\P^2(\xyxy) \nonumber \\
 &   & -\frac{1}{2\lb} \int \bigl[ \Dfive{\flb} (\flb(x)-\fmu(x)) +
                                   \Dsix{\flb} (\flb(\tix)-\fmu(\tix)) \bigr]
                                  \,d\P^2(\xyxy) \,,\nonumber
\eeq
where we used the reproducing property of the kernel, i.e.,
$$
  \langle \Phi(x), f \rangle_H = f(x), \qquad x\in\cX, f\in H,
$$
to obtain the last inequality.
Let us now consider these integrands.
The pairwise loss function $\Ls$ is convex with respect to the last two arguments due to Assumption \ref{sec2.assumption-loss1}.
Hence the convexity yields, for all $x,\tix\in\cX$ and for all $y,\tiy\in\cY$,
\beq
  &  & \Ls(x,y,\tix,\tiy,\tit_1,\tit_2) - \Ls(x,y,\tix,\tiy,t_1,t_2) \nonumber \\
  & \ge & D_5 \Ls(x,y,\tix,\tiy,t_1,t_2) \cdot (\tit_1-t_1)
          + D_6 \Ls(x,y,\tix,\tiy,t_1,t_2) \cdot (\tit_2-t_2) . \nonumber
\eeq
Therefore,
\beq
  & & \Dfive{\fmu} \cdot(\flb(x)-\fmu(x)) +
      \Dsix{\fmu} \cdot (\flb(\tix)-\fmu(\tix)) \nonumber \\
  & \le & L(x,y,\tix,\tiy,\flb(x),\flb(\tix)) - L(x,y,\tix,\tiy,\fmu(x),\fmu(\tix))
 \nonumber
\eeq
and
\beq
  & & \Dfive{\flb} \cdot (\fmu(x)-\flb(x)) +
      \Dsix{\flb} \cdot (\fmu(\tix)-\flb(\tix)) \nonumber \\
  & \le & \Ls(x,y,\tix,\tiy,\fmu(x),\fmu(\tix)) - \Ls(x,y,\tix,\tiy,\flb(x),\flb(\tix)).
 \nonumber
\eeq
If we combine these inequalities and plug them into the above equation, we obtain
\beq
  0 & \le & \hhnorm{\flb-\fmu} \label{appendixc.lem2f3} \\
    & \le & \frac{1}{2\mu} \int \Ls(x,y,\tix,\tiy,\flb(x),\flb(\tix)) -
                                \Ls(x,y,\tix,\tiy,\fmu(x),\fmu(\tix)) \,d\P^2(\xyxy) \nonumber \\
    &   &  + \frac{1}{2\lb} \int \Ls(x,y,\tix,\tiy,\fmu(x),\fmu(\tix)) -
                                 \Ls(x,y,\tix,\tiy,\flb(x),\flb(\tix)) \,d\P^2(\xyxy) \nonumber \\
    & = & \frac{1}{2\mu} \Lsrisk{\flb} - \frac{1}{2\mu} \Lsrisk{\fmu}
          + \frac{1}{2\lb} \Lsrisk{\fmu} - \frac{1}{2\lb} \Lsrisk{\flb} \nonumber \\
    & = & \Bigl( \frac{1}{2\mu} - \frac{1}{2\lb} \Bigr) \Lsrisk{\flb}
          + \Bigl( \frac{1}{2\lb} - \frac{1}{2\mu} \Bigr) \Lsrisk{\fmu}  \nonumber \\
    & = & \Bigl( \frac{1}{2\lb} - \frac{1}{2\mu} \Bigr)
           \bigl( \Lsrisk{\fmu} - \Lsrisk{\flb} \bigr) \,, \label{appendixc.lem2f4}
\eeq
where we used the standard notation for the $\Ls$-risk with respect to a pairwise loss function $\Ls$, i.e.
\be \nonumber
  \Lsrisk{f} := \int \Ls\bigl(x,y,\tix,\tiy,f(x),f(\tix)\bigr) \,d\P^2(\xyxy),
  \qquad f \in H.
\ee
Because we assumed without loss of generality, that $0 < \mu < \lb$, i.e.
$\frac{1}{2\lb}-\frac{1}{2\mu} < 0$, we obtain from
{(\ref{appendixc.lem2f3})}--{(\ref{appendixc.lem2f4})}, that
\be \label{appendixc.lem2f5}
  \Lsrisk{\fmu} \le  \Lsrisk{\flb}.
\ee
Because $\flb$ and  $\fmu$ are elements of $H$, the definition of $\flb$ yields that
$$
  \Lsrisk{\flb} + \lb \hhnorm{\flb}  \le \Lsrisk{\fmu} + \lb \hhnorm{\fmu}\,.
$$
Hence we obtain from {(\ref{appendixc.lem2f5})} and after rearranging the terms in the above inequality that
\be \label{appendixc.lem2f6}
  0 \le \Lsrisk{\flb} - \Lsrisk{\fmu}
   \le  \lb \bigl( \hhnorm{\fmu}  - \hhnorm{\flb} \bigr)
   =     \lb \bigl( \hnorm{\fmu}  + \hnorm{\flb} \bigr) \cdot
               \bigl( \hnorm{\fmu}  - \hnorm{\flb} \bigr).
\ee
Therefore, $\hnorm{\fmu}  - \hnorm{\flb} \ge 0$ and the triangle inequality yields
\be \label{appendixc.lem2f7}
  0 \le \hnorm{\fmu}  - \hnorm{\flb} = \bigl| \hnorm{\fmu}  - \hnorm{\flb} \bigr|
  \le \hnorm{\fmu - \flb}.
\ee
If $\flb=\fmu$, the assertion of the lemma is obviously true. Hence, we may assume
that $\flb \ne \fmu$.
In this case we may divide by the positive term $\hnorm{\flb-\fmu}$.
If we combine {(\ref{appendixc.lem2f3})}-{(\ref{appendixc.lem2f4})} with {(\ref{appendixc.lem2f6})} and {(\ref{appendixc.lem2f7})}, we obtain
\beq
  \hnorm{\flb-\fmu}
  & \le &  \Bigl( \frac{1}{2\lb} - \frac{1}{2\mu} \Bigr)
           \bigl( \Lsrisk{\fmu} - \Lsrisk{\flb} \bigr)
           \cdot \frac{1}{\hnorm{\flb-\fmu}} \nonumber \\
  & = &  \Bigl( \frac{1}{2\mu} - \frac{1}{2\lb} \Bigr)
           \cdot \bigl( \Lsrisk{\flb} - \Lsrisk{\fmu} \bigr)
           \cdot \frac{1}{\hnorm{\flb-\fmu}} \nonumber \\
  & \le &   \Bigl( \frac{1}{2\mu} - \frac{1}{2\lb} \Bigr)
           \cdot \lb \cdot \bigl( \hnorm{\flb} + \hnorm{\fmu} \bigr) \nonumber \\
  & =  &   \frac{1}{2} \, \Bigl( \frac{\lb}{\mu} - 1 \Bigr)
           \cdot \bigl( \hnorm{\flb} + \hnorm{\fmu} \bigr). \nonumber
\eeq
This gives the first assertion for the case $\lb:=\lb_1 > \lb_2=:\mu$,
because $\hnorm{f} \le \inorm{k} \hnorm{f}$ for all $f \in H$.
Of course we can change the roles of $\lb_1$ and $\lb_2$.

We will now show the second assertion. Hence we assume the existence of
a positive constant $r$ with $0<r<\min\{\lb_1,\lb_2\}$.
Using the inequalities (B.12) in (B.13) from
\citet[Lemma B.9]{ChristmannZhou2016a}, we obtain
$$
  \hhanorm{\fbaa}
  \stackrel{\scriptsize{(CZ.(B.12))}}{\le}
  \frac{|L|_1}{\lb_1} \Ex_{\P_{\cX}}|\fbaa(X)|
  \le
  \frac{|L|_1}{\lb_1} \inorm{\fbaa}
  \stackrel{\scriptsize{(CZ.(B.13))}}{\le}
  \frac{1}{\lb_1^2} |L|_1^2 \inorm{k_1}^2
$$
and therefore
$$
  \hanorm{\fbaa} \le \frac{1}{\lb_1} |L|_1 \inorm{k_1} \, .
$$
Of course we obtain with the same argumentation that
$\hanorm{\fbba} \le \frac{1}{\lb_2} |L|_1 \inorm{k_1}$.
For brevity, let us define $\lb_{min}:=\min\{\lb_1,\lb_2\}$
and $\lb_{max}:=\max\{\lb_1,\lb_2\}$.
If we combine {(\ref{appendixc.lem3f1})} with these inequalities, we obtain
\beq
   & & \inorm{\fbaa - \fbba} \nonumber \\
   & \stackrel{\scriptsize{(\ref{appendixc.lem3f1})}}{\le} &
   \frac{\inorm{k_1}}{2} \, \Bigl(\frac{\lb_{max}}{\lb_{min}} - 1\Bigr)  \cdot \bigl(\hanorm{\fbaa} + \hanorm{\fbba}\bigr) \nonumber \\
   & \le &
    \frac{\inorm{k_1}}{2} \, \frac{\lb_{max} - \lb_{min}}{\lb_{min}}  \cdot
\Bigl( \frac{1}{\lb_1} + \frac{1}{\lb_2} \Bigr) |L|_1 \inorm{k_1}
   =
   \frac{|L|_1 \inorm{k_1}^2}{2} \, \bigl(\lb_{max} - \lb_{min}\bigr)
   \frac{\lb_2 + \lb_1}{\lb_{min} \lb_1 \lb_2}   \nonumber \\
   & \le &
   \frac{|L|_1 \inorm{k_1}^2}{2} \, \bigl(\lb_{max} - \lb_{min}\bigr)
   \frac{2 \lb_{max}}{\lb_{min}^2 \lb_{max}}
   \le
   \frac{|L|_1 \inorm{k_1}^2}{r^2} \cdot \bigl|\lb_1 - \lb_2 \bigr| \, ,
 \nonumber
\eeq
which yields the second assertion of the lemma.
\end{proof}


\begin{lemma}\label{appendixc.lem4}
Let Assumptions \ref{sec2.assumption-spaces1},
\ref{sec2.assumption-kernel1}, and \ref{sec2.assumption-loss1} be satisfied.
Define $\kappa=\max\{\inorm{k_1}, \inorm{k_2}\}$ and $d_L:=|D_5 \Ls|_1 + |D_6 \Ls|_1$.
Then,
for all $\lb_2 > \kappa^2 d_L$,
\beq
   \inorm{\fbba - \fbbb}
   \le
   \frac{c_{L,1}}{\lb_2-\kappa^2 d_L} \, \sup_{x\in\cX} \bigl(\inorm{k_2(\cdot,x) - k_1(\cdot,x)} \bigr). \nonumber
\eeq
\end{lemma}

\begin{proof}
To shorten the notation, we will use the following abbrevations in this proof:
$\P:=\P_2$, $\lb:=\lb_2$, $f_1:=\fbba$, and $f_2:=\fbbb$.
We denote the canonical feature maps of the kernels $k_1$ and $k_2$ by $\Phi_1(x):=k_1(\cdot,x)$ and
$\Phi_2(x):=k_2(\cdot,x)$, $x\in\cX$, respectively.
Furthermore we write the partial derivatives of $\Ls$ with respect to the
$i^{th}$ argument at the point $(x,y,\tix,\tiy,f(x),f(\tix))$ by
\be
   D_i \Ls \circ f(x,y,\tix,\tiy) := D_i \Ls(x,y,\tix, \tiy, f(x), f(\tix)),
   \quad i\in\{5,6\}, f\in H.   \nonumber
\ee
By the representer theorem for pairwise loss functions, see Theorem \ref{sec2.representerthm}, we have
\beq
  &   & 2\lb \, \inorm{f_1-f_2} \label{appendixc.lem4f2}  \nonumber\\
  & = & \Big\|
        \int \bigl[ \Dfive{f_1}  \Phi_1(x) + \Dsix{f_1} \Phi_1(\tix) \bigr] \,d\P^2(\xyxy)
         \nonumber \\
  &   & ~~~ - \int \bigl[ \Dfive{f_2}  \Phi_2(x) + \Dsix{f_2} \Phi_2(\tix) \bigr] \,d\P^2(\xyxy) \Big\|_\infty \nonumber \\
  & \le & \int \Bigl\|  \Big( \Dfive{f_2}  \Phi_2(x) + \Dsix{f_2} \Phi_2(\tix)
         \nonumber \\
  &   & ~~~~~~ -  \Dfive{f_1} \Phi_1(x) - \Dsix{f_1} \Phi_1(\tix) \Big) \Big\|_\infty
        \,d\P^2(\xyxy)  \nonumber \\
  & = &  \int \Bigl\|  \Big( \Dfive{f_2} \Phi_2(x) - \Dfive{f_1} \Phi_2(x) \Big) \nonumber \\
  &   & ~~~~~~   + \Big( \Dfive{f_1} \Phi_2(x) - \Dfive{f_1} \Phi_1(x) \Big)   \nonumber \\
  &   & ~~~~~~   + \Big( \Dsix{f_2} \Phi_2(\tix) - \Dsix{f_1} \Phi_2(\tix) \Big) \nonumber \\
  &   & ~~~~~~   + \Big( \Dsix{f_1} \Phi_2(\tix) - \Dsix{f_1} \Phi_1(\tix) \Big) \Big\|_\infty
        \,d\P^2(\xyxy).  \nonumber \
\eeq
It follows that 
\beq
  &   & 2\lb \, \inorm{f_1-f_2}   \nonumber\\
 & \le &
          \int \Bigl\|  \big( \Dfive{f_2} - \Dfive{f_1} \big) \cdot \Phi_2(x)  \nonumber \\
  &   & ~~~~~~   + \big( \Dsix{f_2}  - \Dsix{f_1}  \big) \cdot \Phi_2(\tix) \Big\|_\infty
        \,d\P^2(\xyxy)  \nonumber \\
  &   & +  \int \Bigl\|  \Dfive{f_1} \cdot \big( \Phi_2(x) - \Phi_1(x) \big)   \nonumber \\
  &   & ~~~~~~~~    + \Dsix{f_1} \cdot \big(\Phi_2(\tix) - \Phi_1(\tix) \big) \Big\|_\infty
        \,d\P^2(\xyxy)  \nonumber \\
  & \le &   \sup_{x,\tix\in\cX, y,\tiy\in\cY}
                  \big| \Dfive{f_2} - \Dfive{f_1} \big| \cdot \sup_{x\in\cX} \bigl(\inorm{\Phi_2(x)} \bigr)   \nonumber \\
  &   &  +  \sup_{x,\tix\in\cX, y,\tiy\in\cY}
                  \big| \Dsix{f_2} - \Dsix{f_1} \big| \cdot \sup_{\tix\in\cX} \bigl(\inorm{\Phi_2(\tix)} \bigr)  \nonumber \\
  &   & +  \sup_{x,\tix\in\cX, y,\tiy\in\cY}
                  \big| \Dfive{f_1} \big| \cdot \sup_{x\in\cX} \bigl(\inorm{\Phi_2(x)-\Phi_1(x)} \bigr)   \nonumber \\
  &   &  +  \sup_{x,\tix\in\cX, y,\tiy\in\cY}
                  \big| \Dsix{f_1} \big| \cdot \sup_{\tix\in\cX} \bigl(\inorm{\Phi_2(\tix)-\Phi_1(\tix)} \bigr)   \,. \label{appendixc.lem4f4}  \nonumber
\eeq
Now we can use the assumption that the partial derivatives of $\Ls$ with respect to the fifth and to the sixth argument are Lipschitz continuous with constants $|D_5 \Ls|_1$ and
$|D_6 \Ls|_1$, respectively.
Recall that $\inorm{\Phi(x)} \le \inorm{k}^2$ for all $x\in\cX$.
If we combine this with the uniform boundedness of the partial derivatives of $\Ls$, see Assumption \ref{sec2.assumption-loss1}, we obtain
\beq
  &   & 2\lb \inorm{f_1-f_2} \label{appendixc.lem4f5} \nonumber\\
  & \le & |D_5 \Ls|_1 \cdot \Big( \sup_{x\in\cX} |f_2(x)-f_1(x) | +
                                  \sup_{\tix\in\cX} |f_2(\tix)-f_1(\tix) | \Big) \cdot \inorm{k_2}^2  \nonumber \\
  &   & +  |D_6 \Ls|_1 \cdot \Big( \sup_{x\in\cX} |f_2(x)-f_1(x) | +
                                       \sup_{\tix\in\cX} |f_2(\tix)-f_1(\tix) | \Big) \cdot \inorm{k_2}^2  \nonumber \\
  &   &  +  c_{L,1} \cdot \sup_{x\in\cX} \bigl(\inorm{\Phi_2(x)-\Phi_1(x)} \bigr)   ~ + ~  c_{L,1} \cdot \sup_{\tix\in\cX} \bigl(\inorm{\Phi_2(\tix)-\Phi_1(\tix)} \bigr)  \nonumber \\
  & \le & 2 \inorm{k_2}^2 \cdot \bigl( |D_5 \Ls|_1 + |D_6 \Ls|_1 \bigr) \cdot \inorm{f_2-f_1} +
          2 c_{L,1} \sup_{x\in\cX} \bigl(\inorm{\Phi_2(x)-\Phi_1(x)} \bigr)  \,. \label{appendixc.lem4f6} \nonumber
\eeq
Note that the term $\inorm{f_2-f_1}$ is contained on both sides of the above inequality.
Therefore, if the term
$1 - \frac{1}{\lb}\inorm{k_2}^2 \cdot  \bigl( |D_5 \Ls|_1 + |D_6 \Ls|_1 \bigr)$ is positive, which is
equivalent to $\lb > \inorm{k_2}^2 \cdot  \bigl( |D_5 \Ls|_1 + |D_6 \Ls|_1 \bigr)$, we obtain after division by the factor $2\lb$ and by rearranging terms that
\beq
   \inorm{f_1-f_2} & \le & \frac{\frac{1}{\lb} c_{L,1} \sup_{x\in\cX} \bigl(\inorm{\Phi_2(x)-\Phi_1(x)} \bigr)}
               {1 - \frac{1}{\lb}\inorm{k_2}^2 \cdot  \bigl( |D_5 \Ls|_1 + |D_6 \Ls|_1 \bigr)}\nonumber \\
  & = & \frac{c_{L,1}}{\lb - \inorm{k_2}^2 \cdot  \bigl( |D_5 \Ls|_1 + |D_6 \Ls|_1 \bigr)}
        \cdot \sup_{x\in\cX} \bigl(\inorm{\Phi_2(x)-\Phi_1(x)} \bigr)\,, \nonumber
\eeq
which yields the assertion, if
$\lb:= \lb_2 > \inorm{k_2}^2 \cdot  \bigl( |D_5 \Ls|_1 + |D_6 \Ls|_1 \bigr)$.
Obviously, the kernels $k_1$ and $k_2$ can change their roles and we obtain an analogous
inequality for the case $\lb:= \lb_2 > \inorm{k_1}^2 \cdot  \bigl( |D_5 \Ls|_1 + |D_6 \Ls|_1 \bigr)$.
This gives the assertion.
\end{proof}


\begin{proofof}{\textbf{Proof of Theorem \ref{sec2.thm1}}}
An application of the triangle inequality allows us to use the following error decomposition
\beq
 & & \inorm{\faaa - \fbbb} \nonumber \\
  &   \le   &
  \inorm{\faaa - \fbaa} + \inorm{\fbaa - \fbba} + \inorm{\fbba - \fbbb}\,. \nonumber
\eeq
An application of Lemma \ref{appendixc.lem2}, Lemma \ref{appendixc.lem3}, and Lemma \ref{appendixc.lem4} yields the assertion.
\qedr
\end{proofof}


 \begin{proofof}{\textbf{Proof of Corollary \ref{sec2.cor1}}}
The inequality {(\ref{sec2.cor1F1})} for the difference of the estimated functions follows immediately from the assumption that the positive constant
$s$ is smaller than $\min\{\lb_1,\lb_2\}-r$.

Hence we only have to show the validity of {(\ref{sec2.cor1F2})}.
The proof is very similar to the proof of Corollary \ref{sec1.cor1}.
To shorten the notation in the proof, we define $f_1:=\fsbbb$ and $f_2:=\fsbbb.$
The definitions of $\Lsriska{f_1}$ and $\Lsriskb{f_2}$ yield that
\beqna
& &\Lsriska{f_1}-\Lsriskb{f_2}\\
&=& \int_{(\cXY)^2} L\bigl(x, y, \tix, \tiy, f_1(x), f_1(\tix)\bigr)
  -L(x, y, \tix, \tiy, 0,0)
  \,d\P_1^2(x, y,\tix, \tiy) \\
& &  -
  \int_{(\cXY)^2}  L\bigl(x, y, \tix, \tiy, f_2(x), f_2(\tix)\bigr)
  - L(x, y, \tix, \tiy, 0,0 )
  \, d\P_2^2(x, y,\tix, \tiy).
\eeqna
We plug in the term  $\mp\int_{(\cXY)^2}\int_{\cXY}L\bigl(x, y, \tix, \tiy,f_2(x),f_2(\tix)\bigr)-L(x, y, \tix, \tiy,0,0)\,  d\P_1^2(x,y,\tix, \tiy)$ into the above equation and use the triangle inequality.
Because $L$ is a separately Lipschitz continuous loss function due to Assumption
{\ref{sec2.assumption-loss1}}, we get that
\beqna
& & \big|\Lsriska{f_1}-\Lsriskb{f_2}\big|\nonumber\\
& \le &  \int_{(\cXY)^2} \big|L\bigl(x, y, \tix, \tiy, f_1(x), f_1(\tix)\bigr)
-L\bigl(x, y, \tix, \tiy, f_2(x), f_2(\tix)\bigr)\big|\, d\P_1^2(x, y,\tix,\tiy)
\nonumber\\
& & + \int_{(\cXY)^2} \bigl|L\bigl(x, y, \tix, \tiy, f_2(x), f_2(\tix)\bigr)
  - L(x, y, \tix,\tiy,0,0)\bigr|\, d\bigl(|\P_1-\P_2|^2 \bigr)(x, y,\tix,\tiy)\nonumber\\
&\le & \int_{(\cXY)^2} 2 |L|_{1}|f_1(x)-f_2(x)| \, d\P_1^2(x, y,\tix,\tiy)+
        \int_{(\cXY)^2} 2 |L|_{1}|f_2(x)| \,  d(|\P_1-\P_2|^2)(x, y,\tix,\tiy)\\
&\le & 2 |L|_1 \, \inorm{f_1-f_2} + 2 |L|_1 \, \inorm{f_2}\tvnorm{\P_1-\P_2},\nonumber
\eeqna
where we used in the last step an almost identical argumentation than in the proof of
Lemma \ref{appendixc.lem1}$(iii)$ to get an upper bound for the second integral with respect to
the product measure $|\P_1-\P_2|^2$.

Now we use \citet[Lemma B.9, (B.12)]{ChristmannZhou2016a} and obtain
\beqna
  \inorm{f_2} & \le & \inorm{k_2} \hbnorm{f_2} \\
  & \le & \inorm{k_2} \sqrt{(2/\lb_2) |L|_1 \Ex_{\P_{2_X}} |f_2(X)|} \\
  & \le & \inorm{k_2} \sqrt{(2/\lb_2) |L|_1 \inorm{f_2}} < \infty.
\eeqna
Hence
\beqna
   \inorm{f_2} & \le &  \frac{2}{\lb_2} \, |L|_1 \, \inorm{k_2}^2 \, .
\eeqna
If we combine this inequality with the assumption
$\min\{\lb_1,\lb_2\} > \kappa^2\cdot( |D_5 \Ls|_1 + |D_6 \Ls|_1)$,
we obtain
$$
  \inorm{f_2} \le  \frac{2}{\min\{\lb_1,\lb_2\}}|L|_1
  \max\{\inorm{k_1}^2,\inorm{k_2}^2\}
  \le \frac{2|L|_1}{|D_5 \Ls|_1 + |D_6 \Ls|_1}
  = \frac{2|L|_1}{d_L} .
$$
This upper bound is a constant independent of $\P_1, \P_2, \lb_1, \lb_2, k_1$,
and $k_2$.
If we now combine our inequalities with {(\ref{sec2.cor1F1})}, we obtain
\beqna
& & \big|\Lsriska{f_1}-\Lsriskb{f_2}\big|\nonumber\\
&\le & 2 |L|_1 \, \inorm{f_1-f_2} + 2 |L|_1 \, \inorm{f_2}\tvnorm{\P_1-\P_2},\nonumber \\
& \le & 2 |L|_1 \Bigl( C_1(L)  \cdot \tvnorm{\P_1-\P_2}
          +  C_2(L) \cdot |\lb_1 - \lb_2|
          +  \bar{C}_3(L) \cdot \sup_{x\in\cX} \bigl(\inorm{k_2(\cdot,x) - k_1(\cdot,x)} \Bigr) \nonumber \\
& &  + 2 |L|_1 \, \frac{2|L|_1}{d_L} \tvnorm{\P_1-\P_2}\, .
\eeqna
The desired inequality {(\ref{sec2.cor1F2})} follows by rearranging the factors to compute the constants.
\qedr
\end{proofof}



Now we are in a position to bound the $H_1$-norm of the difference $\fsaaa-\fsbbb$ if we assume $H_2\subseteq H_1.$ We use the triangle inequality to obtain the following decomposition:
\be\label{appendixc.decomphnorm}
\hanorm{\fsaaa-\fsbbb}\leq \hanorm{\fsaaa-\fsbaa}+\hanorm{\fsbaa-\fsbba}+\hanorm{\fsbba-\fsbbb}.
\ee
To bound (\ref{appendixc.decomphnorm}), we first bound the third norm for the case of a differentiable loss.

\begin{lemma}\label{appendixc.diffker}
		Let Assumptions \ref{sec2.assumption-spaces1}, \ref{sec2.assumption-kernel1}, and \ref{sec2.assumption-loss1} be satisfied. Assume that $H_1$ and $H_2$ satisfy (\ref{sec1.inclusion}), then for all $\lb_2>0,$ we have
			$$\hanorm{\fsbba-\fsbbb}\leq \frac{c_{L, 1}}{\lb_2}\sup_{x\in \cX}\hanorm{\ka{x}-\kb{x}}.$$
	\end{lemma}

\begin{proof}
To shorten the notation, we will use the following abbrevations in this proof: $\lb:=\lb_2,$ $\P:=\P_2,$ $f_1:=\fsbba$ and $f_2:=\fsbbb.$ We write the partial derivatives of $\Ls$ with respect to the $i^{th}$ argument by
\be
   D_i \Ls \circ f(x,y,\tix,\tiy) := D_i \Ls(x,y,\tix, \tiy, f(x), f(\tix)),
   \quad i\in\{5,6\}. \nonumber
\ee

Since $H_2\subseteq H_1,$ then $f_2\in H_1.$ It means that $\hanorm{f_1- f_2}$ is well-defined. The representer theorem, see Theorem \ref{sec2.representerthm}, tells us that, for all $\P\in\PM(\cXY),$
\beqnal
	& &\hhanorm{f_{1}-f_{2}}\nonumber\\
	&=& \langle f_{1}-f_{2},\ f_{1}-f_{2} \rangle_{H_1}\nonumber\\
	&=&  \Big\langle f_{1}-f_{2},\  -\frac{1}{2\lb}\int \bigl[ \Dfive{f_1}  \Phi_1(x) + \Dsix{f_1} \Phi_1(\tix) \bigr] \,d\P^2(\xyxy)\nonumber\\
	& &~~+\frac{1}{2\lb}\int \bigl[ \Dfive{f_2}  \Phi_2(x) + \Dsix{f_2} \Phi_2(\tix) \bigr] \,d\P^2(\xyxy)\Big\rangle_{H_1}.\label{appendixc.hdifference}
	\eeqnal
Plugging a zero term into the last term of the above inner product, we know that
	\beqna
	& & -\frac{1}{2\lb}\int \bigl[ \Dfive{f_1}  \Phi_1(x) + \Dsix{f_1} \Phi_1(\tix) \bigr] \,d\P^2(\xyxy)\\
	& &+\frac{1}{2\lb}\int \bigl[ \Dfive{f_2}  \Phi_2(x) + \Dsix{f_2} \Phi_2(\tix) \bigr] \,d\P^2(\xyxy)\\
&=&  -\frac{1}{2\lb}\int \bigl[ \Dfive{f_1}  \Phi_1(x) + \Dsix{f_1} \Phi_1(\tix) \bigr] \,d\P^2(\xyxy)\\
& &+\frac{1}{2\lb}\int \bigl[ \Dfive{f_2}  \Phi_1(x) + \Dsix{f_2} \Phi_1(\tix) \bigr] \,d\P^2(\xyxy)\\
	& & -\frac{1}{2\lb}\int \bigl[ \Dfive{f_2}  \Phi_1(x) + \Dsix{f_2} \Phi_1(\tix) \bigr] \,d\P^2(\xyxy)\\
	& &+\frac{1}{2\lb}\int \bigl[ \Dfive{f_2}  \Phi_2(x) + \Dsix{f_2} \Phi_2(\tix) \bigr] \,d\P^2(\xyxy).
	\eeqna
Applying the reproducing property (\ref{sec1.reproducingproperty}), we obtain from (\ref{appendixc.hdifference}) and the above equation that
\beqnal\label{appendixc.1}
& &\hhanorm{f_{1}- f_{2}}\nonumber\\
	&= &-\frac{1}{2\lb} \int \Bigl[\bigl( \Dfive{f_1}-\Dfive{f_2}\bigr)(f_1(x)-f_2(x))\\
& & ~~~~~~~~~~~+ \bigl( \Dsix{f_1}-\Dsix{f_2} (f_1(\tix)-f_2(\tix))\bigr)\Bigr] \,d\P^2(\xyxy) \nonumber\\
	& &- \frac{1}{2\lb}\Big\langle \int \Bigl[\Dfive{f_2}(\Phi_1(x)-\Phi_2(x))\nonumber\\
& &~~~~~~~~~~~~~+\Dsix{f_2}(\Phi_1(\tix)-\Phi_2(\tix))\Bigr] \,d\P^2(\xyxy),\  f_{1}- f_{2}\Big\rangle_{H_1}.\nonumber
	\eeqnal
The convexity of the pairwise loss function $\Ls$ with respect to the last two arguments implies that
for all $x,\tix\in\cX$ and for all $y,\tiy\in\cY$,
\beq
  &  & \Ls(x,y,\tix,\tiy,\tit_1,\tit_2) - \Ls(x,y,\tix,\tiy,t_1,t_2) \nonumber \\
  & \ge & D_5 \Ls(x,y,\tix,\tiy,t_1,t_2) \cdot (\tit_1-t_1)
          + D_6 \Ls(x,y,\tix,\tiy,t_1,t_2) \cdot (\tit_2-t_2) . \nonumber
\eeq
and
\beq
  &  & \Ls(x,y,\tix,\tiy,t_1,t_2) - \Ls(x,y,\tix,\tiy,\tit_1,\tit_2) \nonumber \\
  & \ge & D_5 \Ls(x,y,\tix,\tiy,\tit_1,\tit_2) \cdot (t_1-\tit_1)
          + D_6 \Ls(x,y,\tix,\tiy,\tit_1,\tit_2) \cdot (t_2-\tit_2) . \nonumber
\eeq
Adding both sides of above two inequalities, we get that
\beq
& &\bigl[D_5 \Ls(x,y,\tix,\tiy,t_1,t_2)-D_5 \Ls(x,y,\tix,\tiy,\tit_1,\tit_2)\bigr]\cdot (\tit_1-t_1)\nonumber\\
& & +\bigl[D_6 \Ls(x,y,\tix,\tiy,t_1,t_2) -D_6 \Ls(x,y,\tix,\tiy,\tit_1,\tit_2)\bigr]\cdot (\tit_2-t_2)\leq 0.\nonumber
\eeq
Taking $\tit_1=f_1(x), t_1=f_2(x), \tit_2=f_1(\tix),$ and $t_2=f_2(\tix),$ then we know that the integrand in (\ref{appendixc.1})$\leq 0.$
It follows from (\ref{loss-assump1}) that
\beqna
\hhanorm{f_{1}- f_{2}}&\leq & - \frac{1}{2\lb}\Big\langle \int \Bigl[\Dfive{f_2}(\Phi_1(x)-\Phi_2(x))\\
& &~~~~~~~~~~~~~ +\Dsix{f_2}(\Phi_1(\tix)-\Phi_2(\tix))\Bigr] \,d\P^2(\xyxy),\  f_{1}- f_{2}\Big\rangle_{H_1}\\
&\leq & \frac{1}{2\lb}\cdot 2c_{L, 1}\cdot \hanorm{f_{1}- f_{2}}\cdot\sup_{x\in \cX}\hanorm{\Phi_1(x)-\Phi_2(x)}.
	\eeqna
The desired result comes immediately from above inequality.
\end{proof}

\begin{lemma}\label{appendixc.differentiable}
		Let Assumptions \ref{sec2.assumption-spaces1}, \ref{sec2.assumption-kernel1}, and \ref{sec2.assumption-loss1} be satisfied. Assume that $H_1$ and $H_2$ satisfy (\ref{sec1.inclusion}), then for all $\lb_1, \lb_2>0,$ we have
			\beqnal\label{sec1.hnorm1}
	& &\hanorm{\faaa-\fbbb}\nonumber\\
	&\leq & C_1^\prime(L, \lb_1, \lb_2)\tvnorm{\P_1-\P_2}+C_2^\prime(L, \lb_1, \lb_2)|\lb_1 -\lb_2|+C_3^\prime(L, \lb_1, \lb_2)\sup_{x\in \cX}\hanorm{\ka{x}-\kb{x}},\nonumber
	\eeqnal
	where $C_1^\prime(L, \lb_1, \lb_2):=\frac{4\kappa c_{L, 1}}{\min\{\lb_1, \lb_2\}}, C_2^\prime(L, \lb_1, \lb_2):=\frac{\kappa|L|_1}{\min\{\lb_1^2, \lb_2^2\}},$ and $C_3^\prime(L, \lb_1, \lb_2):=\frac{c_{L, 1}}{\min\{\lb_1, \lb_2\}}.$	
	\end{lemma}

\begin{proof}
An application of the triangle inequality (\ref{appendixc.decomphnorm}), Lemma \ref{appendixc.lem2}, Lemma \ref{appendixc.lem3} and Lemma \ref{appendixc.diffker} yields the assertion by using $\inorm{f}\leq \inorm{k_1}\cdot \hanorm{f},$ $\forall f\in H_1.$
\end{proof}


\begin{proofof}{\textbf{Proof of Theorem  \ref{sec2.thm2}}}
(i) Lemma \ref{appendixc.differentiable} shows the first assertion.

(ii) We just need to prove the assertion holds true for the case of a non-differentiable loss, where the pairwise loss function $L$ can be represented by a convex and Lipschitz continuous function $\rho:\R\to\R$, see {(\ref{loss:def-mee})}, i.e. we have
$L(x,y,\tix,\tiy,t,\tit) := \rho\bigl((y-t) - (\tiy-\tit) \bigr)$.

Step 1: We contruct a differentiable approximator $\rho_{\d}$ for the non-differentiable $\rho$ by smoothing $\rho$ by convolution with the uniform distribution on the interval $[-\d,0]$, where $\d\in(0,1]$, see e.g.
\citet[p.148]{CheneyLight2000}.

Let $\rho$ be convex, but non-differentiable. Let $0<\d\leq 1$.
Then we can approximate it by the function
$$
  \rho_\d: \R\to\R, \quad \rho_{\d}(\xi)
  =
  \int_{0}^{1}\rho(\xi-\d\theta)d\theta=\frac{1}{d}\int_{\xi-\d}^{\xi}\rho(u)du.
$$
It is easy to check that $\rho_{\d}$ is convex, differentiable and Lipschitz continuous with constant $|\rho|_1$.
The derivative of $\rho_{\d}$ equals
\begin{equation}\label{appendixc.derivative}
  \rho_{\d}^{\prime}(\x)=\frac{1}{\d}\bigl(\rho(\xi)-\rho(\xi-\d)\bigr), \quad \xi\in\R,
\end{equation}
which can be bounded by
\be\label{appendixc.boundofderivative}
\inorm{\rho_{\d}^{\prime}}\leq\sup_{\xi\in\R}\Big|\frac{1}{\d}\bigl(\rho(\xi)-\rho(\xi-\d)\bigr)\Big|\leq \frac{1}{\d}\cdot|\rho|_1\cdot\d\leq |\rho|_1.
\ee
The approximation is valid, because for every $\xi\in\R,$
$$
  |\rho_{\d}(\xi)-\rho(\xi)|
  =
  \Bigl|\int_{0}^{1}\rho(\xi-\d \theta)-\rho(\xi)\,d\theta\Bigr|
  \leq \int_{0}^{1}|\rho|_1\d \theta\,d\theta\leq \frac{|\rho|_1}{2}\d.
$$
Hence
\begin{equation}\label{appendixc.uniformconvergence}
  \|\rho_{\d}-\rho\|_{\infty}=\mathcal{O}(\d),\ \hbox{as}\, \, \d\to 0_+.
\end{equation}
In order to avoid any moment conditions on the probability measure, we define a shifted version of $\rho_{\d}$ by
$$
  \rho^\star_{\d}\big((y-f(x))-(\tilde y-f(\tilde x))\big)
  =
  \rho_{\d}\big((y-f(x))-(\tilde y-f(\tilde x))\big)-\rho_{\d}(y-\tilde y),
$$
which is convex, differentiable and Lipschitz continuous with constant $|\rho|_1$ as well. Obviously, $\rho^\star_{\d}$ and $\rho_{\d}$ has the same derivative, then it comes immediately from (\ref{appendixc.boundofderivative})that
\be\label{appendixc.shiftedderivativebound}
\inorm{(\rho^\star_{\d})'}=\inorm{\rho'_{\d}}\leq |\rho|_1.
\ee
Hence we will approximate the convex, but non-differentiable shifted pairwise loss function $\rho^\star$  by the convex, differentiable shifted pairwise loss function $\rho^\star_{\d}.$

Let us define the $\rho^\star_{\d}$-risk, the regularized $\rho^\star_{\d}$-risk and the regularizing function $f_{\P, \lb, k, (\d)}$  for any $0\leq \d\leq 1$ as below:
\beqna
& & \cR_{\rho^\star_\d, \P}(f):=\E_{\P^2}\rho^\star_\d\big((Y-f(X))-(\tilde Y-f(\tilde X))\big),\\
& & \cR_{\rho^\star_\d, \P, \lb}(f):=\cR_{\rho^\star_\d, \P}(f)+\lb\hnorm{f}^2,\\
& & f_{\P, \lb, k, (\d)}:=\arg\inf_{f\in \cH} \cR_{\rho^\star_\d, \P, \lb, }(f).
\eeqna

Step 2: We now show the weak convergence in $\cH$ of $f_{\P, \lb, k, (\d_j)}$ to $f_{\P, \lb, k}:=\arg\inf_{f\in \cH} \cR_{\rho^\star, \P, \lb}(f),$ for $(\d_j)_{j\in \N}$ with $\d_j\to 0$ and $\d_j\in(0, 1).$

 \citet[Lemma B.9, (B.12) and (B.13)]{ChristmannZhou2016a} tells us that for any $0\leq \d\leq 1$
\begin{eqnarray}\label{Bound}
\hnorm{f_{\P, \lb, k, (\d)}}&\leq &\sqrt{(1/\lb) |\rho|_1 \Ex_{\P_{X}} |f_{\P, \lb, k, (\d)}(X)|} \notag\\
  & \le &  \sqrt{(1/\lb) |\rho|_1 \inorm{f_{\P, \lb, k, (\d)}}}\notag\\
&\leq &  \sqrt{(1/\lb^2)|\rho|_1^2\inorm{k}^2}=\lb^{-1}|\rho|_1\inorm{k}.
\end{eqnarray}
Any closed ball $B_R=\{f\in \cH\, :\, \hnorm{f}\leq R \}$ of the Hilbert space $\cH$ with a finite radius $R>0$ is weakly compact. Hence the estimate (\ref{Bound}) tells us that there
exists a decreasing sequence $(\d_j)_{j\in\N},$ with $\d_j\in (0, 1)$ such that $\lim_{j\to \infty}\d_j=0$ and $f_{\P, \lb, k, (\d_j)}$ weakly converges to some $g_{\P, \lb, k}\in \cH.$ That is
\begin{equation}\label{weakconvergence}
\lim_{j\to \infty}\langle f_{\P, \lb, k, (\d_j)}, f\rangle_\cH=\langle g_{\P, k, \lb}, f\rangle_\cH, \quad \forall \ f\in\cH.
\end{equation}
Let $f= g_{\P, \lb, k}$ in (\ref{weakconvergence}). Then we obtain by the Cauchy-Schwartz inequality that
\begin{equation*}\label{limitbound}
\hnorm{g_{\P, \lb, k}}^2=\langle g_{\P, k, \lb},  g_{\P, k, \lb}\rangle_\cH=\lim_{j\to \infty}\langle f_{\P, \lb, k, (\d_j)},  g_{\P, \lb, k}\rangle_\cH\leq
\hnorm{g_{\P, \lb, k}} \liminf_{j\to\infty}\|f_{\P, \lb, k, (\d_j)}\|_\cH.
\end{equation*}
Therefore, together with (\ref{Bound}), we get that
\begin{equation}\label{limitbound}
\hnorm{g_{\P, \lb, k}}\leq \liminf_{j\to\infty}\hnorm{f_{\P, \lb, k, (\d_j)}}\leq \lb^{-1}|\rho|_1\inorm{k}.
\end{equation}
Let $x\in\cX$ and $f=\kk{x}$ in (\ref{weakconvergence}). The reproducing property  (\ref{sec1.reproducingproperty}) yields
\begin{equation}\label{pointconverge}
g_{\P, \lb, k}(x)=\langle g_{\P, \lb, k},  \kk{x}\rangle_\cH=\lim_{j\to \infty}\langle f_{\P, \lb, k, (\d_j)}, \kk{x}\rangle_\cH=\lim_{j\to\infty}f_{\P, \lb, k, (\d_j)}(x).
\end{equation}
The Lipschitz continuity of $\rho^\star$ together with (\ref{pointconverge}) tells us that
$$\lim_{j\to \infty}\rho^\star\big((y-f_{\P, \lb, k, (\d_j)}(x))-(\tilde y-f_{\P, \lb, k, (\d_j)}(\tilde x))\big)=\rho^\star\big((y-g_{\P, \lb, k}(x))-(\tilde y-g_{\P, \lb, k}(\tilde x))\big).$$
The Lebesgue Dominated Theorem gives
\begin{equation}\label{riskconvergent}
\ \cR_{\rho^\star, \P }(g_{\P, k, \lb})=\lim_{j\to\infty} \cR_{\rho^\star, \P}(f_{\P, \lb, k, (\d_j)}).
\end{equation}
The uniform estimate (\ref{appendixc.uniformconvergence}) in connection with (\ref{pointconverge}) yields
\begin{eqnarray}\label{riskconvergence1}
\lim_{j\to \infty}\cR_{\rho^\star_{\d_j}, \P}(f_{\P, \lb, k, (\d_j)})&=&\lim_{j\to \infty}\int\rho^\star_{\d_j}\big((y-f_{\P, \lb, k, (\d_j)}(x))-(\tilde y-f_{\P, \lb, k, (\d_j)}(\tilde x))\big)\ d\P^2(x, y, \tilde x, \tilde y)\notag\\
&=& \lim_{j\to \infty}\int\rho^\star\big((y-f_{\P, \lb, k, (\d_j)}(x))-(\tilde y-f_{\P, \lb, k, (\d_j)}(\tilde x))\big)\ d\P^2(x, y, \tilde x, \tilde y)\notag\\
&=&\int\rho^\star\big((y-g_{\P, \lb, k}(x))-(\tilde y-g_{\P, \lb, k}(\tilde x))\big)\ d\P^2(x, y, \tilde x, \tilde y)\notag\\
&=&\cR_{\rho^\star, \P}(g_{\P, \lb, k}).
\end{eqnarray}
Therefore, combining (\ref{limitbound}) with (\ref{riskconvergence1}), we have
\begin{equation*}
\cR_{\rho^\star, \P}(g_{\P, \lb, k})+\lb\hnorm{g_{\P, \lb, k}}^2\leq \liminf_{j\to\infty}\big\{\cR_{\rho^\star_{\d_j}, \P}(f_{\P, \lb, k, (\d_j)})+\lb\hnorm{f_{\P, \lb, k, (\d_j)}}^2\big\}.
\end{equation*}
By the definition of $f_{\P, \lb, k, (\d_j)},$ we obtain
\beqna
\liminf_{j\to\infty}\big\{\cR_{\rho^\star_{\d_j}, \P}(f_{\P, \lb, k, (\d_j)})+\lb\hnorm{f_{\P, \lb, k, (\d_j)}}^2\big\} &\leq &  \liminf_{j\to\infty}\big\{\cR_{\rho^\star_{\d_j}, \P}(f_{\P, \lb, k})+\lb\hnorm{f_{\P, \lb, k}}^2\big\}\\
&=& \cR_{\rho^\star, \P}(f_{\P, \lb, k})+\lb\hnorm{f_{\P, \lb, k}}^2,
\eeqna
which implies that
$$\cR_{\rho^\star, \P}(g_{\P, \lb, k})+\lb\hnorm{g_{\P, \lb, k}}^2\leq \cR_{\rho^\star, \P}(f_{\P, \lb, k})+\lb\hnorm{f_{\P, \lb, k}}^2.$$
The strict convexity of the regularized risk functional $\cR_{\rho^\star, \P, \lb}(\cdot)$ on $\cH$ guarantees the uniqueness of the minimizer, which leads to $g_{\P, \lb, k}=f_{\P, \lb, k}$ and
\begin{equation}\label{appendixc.realweakconvergence}
\lim_{j\to \infty}\langle f_{\P, \lb, k, (\d_j)}, f\rangle_\cH=\langle f_{\P, k, \lb}, f\rangle_\cH, \quad \forall \ f\in\cH.
\end{equation}

Step 3. In the rest of the proof, we focus on estimating $\hanorm{\fsaaa-\fsbbb} .$ \
\beqna
& &\hanorm{\fsaaa-\fsbbb}\\
&\leq & \hanorm{\fsaaa-f_{\P_1, \lb_1, k_1, (\d_j)}}+\hanorm{f_{\P_1, \lb_1, k_1, (\d_j)}-f_{\P_2, \lb_2, k_2, (\d_j)}}+\hanorm{f_{\P_2, \lb_2, k_2, (\d_j)}-\fsbbb}.
\eeqna
The weak convergence (\ref{appendixc.realweakconvergence}) tells us that
$$\lim_{j\to\infty}\langle f_{\P_1, \lb_1, k_1, (\d_j)}, \fsaaa\rangle_{H_1}=\hhanorm{\fsaaa}$$
and
$$\lim_{j\to\infty}\hhanorm{f_{\P_1, \lb_1, k_1, (\d_j)}}=\hhanorm{\fsaaa},$$
which implies that
\beqna
&&\lim_{j\to\infty}\hhanorm{\fsaaa-f_{\P_1, \lb_1, k_1, (\d_j)}}\\
&=&\hhanorm{\fsaaa}+\lim_{j\to\infty}\hhanorm{f_{\P_1, \lb_1, k_1, (\d_j)}}-2\lim_{j\to\infty}\langle f_{\P_1, \lb_1, k_1, (\d_j)}, \fsaaa\rangle_{H_1}=0.
\eeqna

In the same way, we can prove
$$\lim_{j\to\infty}\hanorm{f_{\P_2, \lb_2, k_2, (\d_j)}-\fsbbb}=0.$$

Since $\rho^\star_{\d_j}$ is a convex, differentiable and Lipschitz continuous shifted loss function with constant $|\rho|_1$ and the uniform upper bound (\ref{appendixc.shiftedderivativebound}) for the derivative of $\rho^\star_{\d_j},$ Lemma \ref{appendixc.differentiable} yields that, for all $\lb_1, \lb_2>0,$
\beqnal\label{sec1.hnorm1}
	& &\hanorm{f_{\P_1, \lb_1, k_1, (\d_j)}-f_{\P_2, \lb_2, k_2, (\d_j)}}\nonumber\\
	&\leq & \tilde{C}_1^\prime(\rho, \lb_1, \lb_2)\tvnorm{\P_1-\P_2}+\tilde{C}_2^\prime(\rho, \lb_1, \lb_2)|\lb_1 -\lb_2|+\tilde{C}_3^\prime(\rho, \lb_1, \lb_2)\sup_{x\in \cX}\hanorm{\ka{x}-\kb{x}},\nonumber
	\eeqnal
	where $\tilde{C}_1^\prime(\rho, \lb_1, \lb_2):=\frac{4\kappa |\rho|_1}{\min\{\lb_1, \lb_2\}}, \tilde{C}_2^\prime(\rho, \lb_1, \lb_2):=\frac{\kappa|\rho|_1}{\min\{\lb_1^2, \lb_2^2\}},$ and $\tilde{C}_3^\prime(\rho, \lb_1, \lb_2):=\frac{|\rho|_1}{\min\{\lb_1, \lb_2\}}.$	

Therefore, our desired result is proved.

\qedr
\end{proofof}


 \subsection{Appendix D: Proofs for results in Section \ref{sec3}}\label{appendixd}
   
\begin{proofof}{\textbf{Proof of Corollary \ref{sec3.corollary2}}}
When $m=1,$ we take two kernels $k_{\bf{w}, \g_1}$ and $k_{\tilde{\bf{w}}, \g_1}$ both with depth $1,$ but with different weight parameters $\bf{w}$ and $\tilde{\bf{w}}.$
Notice that the univariate function $g$ given by $g(u) = \exp\left(-2 \g^{-2} u\right)$ satisfies $\max_{u\in [0, +\infty)} |g'(u)| = 2 \g^{-2}$. So we know that for all $x, x' \in \cX$,
\beq
& & \sup_{x\in\cX} \inorm{k_{\bf{w}, \g_1}(\cdot,x)-k_{\tilde{\bf{w}}, \g_1}(\cdot,x)} \nonumber\\
&=&  \sup_{x,x'\in\cX} \big|\exp\big(-2\g_1^{-2}\sum_{i\in I}w_i^2 (x_i-x_i^{\prime})^2\big)-\exp\big(-2\g_1^{-2}\sum_{i\in I}\tilde{w}_i^2 (x_i-x_i^{\prime})^2\big) \big|\nonumber \\
&\leq &  2\g_1^{-2}\sup_{x,x'\in\cX}\big|\sum_{i\in I}w_i^2 (x_i-x_i^{\prime})^2-\sum_{i\in I}\tilde{w}_i^2 (x_i-x_i^{\prime})^2\big|\nonumber\\
&\leq & 2\g_1^{-2} \sum_{i\in I}|w_i^2-\tilde{w}_i^2|\sup_{x_i,x'_i}|x_i-x_i^{\prime}|^2\nonumber\\
&\leq & 2\g_1^{-2}(diam(\cX))^2\sum_{i\in I}|w_i^2-\tilde{w}_i^2|.\nonumber
\eeq

But $\sum_{i\in I} w_i^2 \leq 1$ and $\sum_{i\in I} \tilde{w}_i^2 \leq 1,$ which yield
$$\sum_{i\in I} \left|w_i^2 - \tilde{w}_i^2\right| =\sum_{i\in I} \left|w_i + \tilde{w}_i\right| \left|w_i - \tilde{w}_i\right|
\leq \|{\bf w} + \tilde{{\bf w}}\|_{\ell^2} \|{\bf w} - \tilde{{\bf w}}\|_{\ell^2} \leq 2 \|{\bf w} - \tilde{{\bf w}}\|_{\ell^2}. $$

Hence
$$\sup_{x\in\cX} \inorm{k_{\bf{w}, \g_1}(\cdot,x)-k_{\tilde{\bf{w}}, \g_1}(\cdot,x)}\leq 4\g_1^{-2}(diam(\cX))^2\|{\bf w} - \tilde{{\bf w}}\|_{\ell^2},$$
which leads to our first assertion (\ref{depth1b}).

For the case of $m>1$, recall the definition of depth $m$ hierarchical Gaussian kernel as follows:
$$ k_{{\bf W}^{(1)}, \ldots, {\bf W}^{(m-1)}, {\bf w}, \gamma_1, \ldots, \gamma_m} (x, x')
=\exp\left(-2 \gamma_m^{-2} \sum_{i=1}^\ell w_i^2 \left(1- k_{{\bf W}_i^{(1)}, \ldots, {\bf W}_i^{(m-2)}, {\bf w}^{(m-1)}_i, \gamma_1, \ldots, \gamma_{m-1}} (x_{I_i}, x'_{I_i})\right)\right), $$
and denote the deviation quantity at depth $j \in \{1, \ldots, m\}$ as
$$ \Delta_j =\sup_{x\in {\mathcal X}_I} \sum_i \left\|k_{{\bf W}_i^{(1)}, \ldots, {\bf W}_i^{(j-1)}, {\bf w}^{(j)}_i, \gamma_1, \ldots, \gamma_j} (\cdot, x)
- k_{\tilde{{\bf W}}_i^{(1)}, \ldots, \tilde{{\bf W}}_i^{(j-1)}, {\bf w}^{(j)}_i, \gamma_1, \ldots, \gamma_j} (\cdot, x)\right\|_\infty. $$

Since the norms of the weights satisfy
$$\|{\bf W}^{(j)}\|_{\ell_2}:=\sum_{i=1}^{\ell}\|{\bf w}^{(j)}_i\|_{\ell^2}^2 \leq 1, \qquad j=1, \ldots, m, $$
where ${\bf w}^{(m)} = {\bf w},$ we have
$$ \sum_{i=1}^\ell w_i^2 \left(1- k_{{\bf W}_i^{(1)}, \ldots, {\bf W}_i^{(m-2)}, {\bf w}^{(m-1)}_i, \gamma_1, \ldots, \gamma_{m-1}} (x_{I_i}, x'_{I_i})\right) \in [0, 1]. $$

The univariate function $g$ given by $g(u) = \exp\left(-2 \gamma_m^{-2} u\right)$ satisfies $\max_{u\in [0, 1]} |g'(u)| = 2 \gamma_m^{-2}$. Hence $g$ is Lipschitz continuous. So we know that for $x, x' \in {\mathcal X}_I$,
\begin{eqnarray*}
&& \left|k_{{\bf W}^{(1)}, \ldots, {\bf W}^{(m-1)}, {\bf w}, \gamma_1, \ldots, \gamma_m} (x, x')
- k_{\tilde{{\bf W}}^{(1)}, \ldots, \tilde{{\bf W}}^{(m-1)}, \tilde{{\bf w}}, \gamma_1, \ldots, \gamma_m} (x, x')\right| \\
&\leq& 2 \gamma_m^{-2} \left|\sum_{i=1}^\ell w_i^2 k_{{\bf W}_i^{(1)}, \ldots, {\bf W}_i^{(m-2)}, {\bf w}^{(m-1)}_i, \gamma_1, \ldots, \gamma_{m-1}} (x_{I_i}, x'_{I_i}) - \tilde{w}_i^2 k_{\tilde{{\bf W}}_i^{(1)}, \ldots, \tilde{{\bf W}}_i^{(m-2)}, \tilde{{\bf w}}^{(m-1)}_i, \gamma_1, \ldots, \gamma_{m-1}} (x_{I_i}, x'_{I_i})\right| \\
&\leq& 2 \gamma_m^{-2} \left|\sum_{i=1}^\ell w_i^2 \left(k_{{\bf W}_i^{(1)}, \ldots, {\bf W}_i^{(m-2)}, {\bf w}^{(m-1)}_i, \gamma_1, \ldots, \gamma_{m-1}} (x_{I_i}, x'_{I_i}) - k_{\tilde{{\bf W}}_i^{(1)}, \ldots, \tilde{{\bf W}}_i^{(m-2)}, \tilde{{\bf w}}^{(m-1)}_i, \gamma_1, \ldots, \gamma_{m-1}} (x_{I_i}, x'_{I_i})\right)\right| \\
&& + 2 \gamma_m^{-2} \left|\sum_{i=1}^\ell \left(w_i^2 - \tilde{w}_i^2\right) k_{\tilde{{\bf W}}_i^{(1)}, \ldots, \tilde{{\bf W}}_i^{(m-2)}, \tilde{{\bf w}}^{(m-1)}_i, \gamma_1, \ldots, \gamma_{m-1}} (x_{I_i}, x'_{I_i})\right| \\
&\leq& 2 \gamma_m^{-2} \sum_{i=1}^\ell w_i^2 \Delta_{m-1}  + 2 \gamma_m^{-2} \sum_{i=1}^\ell \left|w_i^2 - \tilde{w}_i^2\right|.
\end{eqnarray*}
But $\sum_{i=1}^\ell w_i^2 \leq 1$ and  $\sum_{i\in I} \tilde{w}_i^2 \leq 1,$ which yield
$$\sum_{i=1}^\ell \left|w_i^2 - \tilde{w}_i^2\right| =\sum_{i=1}^\ell \left|w_i + \tilde{w}_i\right| \left|w_i - \tilde{w}_i\right|
\leq \|{\bf w} + \tilde{{\bf w}}\|_{\ell^2} \|{\bf w} - \tilde{{\bf w}}\|_{\ell^2} \leq 2 \|{\bf w} - \tilde{{\bf w}}\|_{\ell^2}. $$
Hence
$$ \Delta_m \leq 4 \gamma_m^{-2} \|{\bf w} - \tilde{{\bf w}}\|_{\ell^2} + 2 \gamma_m^{-2} \Delta_{m-1}. $$

Notice that $\Delta_1 \leq 2 \gamma_1^{-2} \left(diam({\mathcal X})\right)^2 \|{\bf W}^{(1)} - \tilde{{\bf W}}^{(1)}\|_{\ell^2}$. Then by induction we have
\beq
 \Delta_m &\leq & 4 \gamma_m^{-2} \|{\bf w} - \tilde{{\bf w}}\|_{\ell^2} + \sum_{j=2}^{m-1}
2^{m-j +2} \left(\Pi_{p=j}^m \gamma_{p}^{-2}\right) \|{\bf W}^{(j)} - \tilde{{\bf W}}^{(j)}\|_{\ell^2}\nonumber\\
& & + 2^{m} \left(diam({\mathcal X})\right)^2
\left(\Pi_{p=1}^m \gamma_{p}^{-2}\right) \|{\bf W}^{(1)} - \tilde{{\bf W}}^{(1)}\|_{\ell^2}. \nonumber
\eeq

From this we obtain the assertion (\ref{depthmb}).
\qedr
\end{proofof}


\begin{thebibliography}{}

\bibitem[Aronszajn(1950)Aronszajn]{Aronszajn1950}
Aronszajn, N. (1950).
\newblock Theory of reproducing kernels.
\newblock {\em Trans. Amer. Math. Soc.}, {\bf 68}, 337--404.

\bibitem[Berlinet and {Thomas-Agnan}(2004)Berlinet and
  {Thomas-Agnan}]{BerlinetThomasAgnan2004}
Berlinet, A. and {Thomas-Agnan}, C. (2004).
\newblock {\em Reproducing kernel {H}ilbert spaces in probability and
  statistics\/}.
\newblock Kluwer, Boston.

\bibitem[Billingsley(1999)Billingsley]{Billingsley1999}
Billingsley, P. (1999).
\newblock {\em Convergence of probability measures\/}.
\newblock John Wiley {\&} Sons, New York, 2nd edition.

\bibitem[Bousquet and Elisseeff(2001)Bousquet and
  Elisseeff]{BousquetElisseeff2001}
Bousquet, O. and Elisseeff, A. (2001).
\newblock Algorithmic stability and generalization performance.
\newblock In T.~Leen, T.~Dietterich, and V.~Tresp, editors, {\em Advances in
  Neural Information Processing Systems 13\/}, pages 196--202. MIT Press.

\bibitem[Caponnetto and {De Vito}(2007)Caponnetto and {De
  Vito}]{CaponnettoDevito2007}
Caponnetto, A. and {De Vito}, E. (2007).
\newblock {O}ptimal {R}ates for the {R}egularized {L}east-{S}quares
  {A}lgorithm.
\newblock {\em Found. Comput. Math.}, pages 331--368.

\bibitem[Cheney and Light(2000)Cheney and Light]{CheneyLight2000}
Cheney, W. and Light, W. (2000).
\newblock {\em A Course in Approximation Theory\/}.
\newblock Brooks/Cole Publishing Company, Pacific Grove.

\bibitem[Christmann and Steinwart(2004)Christmann and
  Steinwart]{ChristmannSteinwart2004a}
Christmann, A. and Steinwart, I. (2004).
\newblock On robust properties of convex risk minimization methods for pattern
  recognition.
\newblock {\em J. Mach. Learn. Res.}, {\bf 5}, 1007--1034.

\bibitem[Christmann and Steinwart(2007)Christmann and
  Steinwart]{ChristmannSteinwart2007a}
Christmann, A. and Steinwart, I. (2007).
\newblock Consistency and robustness of kernel based regression.
\newblock {\em Bernoulli\/}, {\bf 13}, 799--819.

\bibitem[Christmann and Zhou(2016)Christmann and Zhou]{ChristmannZhou2016a}
Christmann, A. and Zhou, D.~X. (2016).
\newblock On the robustness of regularized pairwise learning methods based on
  kernels.
\newblock {\em Journal of Complexity\/}, {\bf 37}, 1--33.

\bibitem[Christmann {\em et~al.}(2009)Christmann, {V}an Messem, and
  Steinwart]{ChristmannVanMessemSteinwart2009}
Christmann, A., {V}an Messem, A., and Steinwart, I. (2009).
\newblock On consistency and robustness properties of support vector machines
  for heavy-tailed distributions.
\newblock {\em Statistics and Its Interface\/}, {\bf 2}, 311--327.

\bibitem[Christmann {\em et~al.}(2013)Christmann, Salib\'{i}an-Barrera, and
  Aelst]{ChristmannSalibianBarreraVanAelst2013}
Christmann, A., Salib\'{i}an-Barrera, M., and Aelst, S.~V. (2013).
\newblock Qualitative robustness of bootstrap approximations for kernel based
  methods.
\newblock In C.~Becker, R.~Fried, and S.~Kuhnt, editors, {\em Robustness and
  Complex Data Structures. Festschrift in Honour of Ursula Gather\/}, pages
  263--278. Springer, Heidelberg, New York.

\bibitem[Cucker and Smale(2002)Cucker and Smale]{CuckerSmale2002}
Cucker, F. and Smale, S. (2002).
\newblock On the mathematical foundations of learning.
\newblock {\em Bull. Amer. Math. Soc. (N.~S.)\/}, {\bf 39}, 1--49.

\bibitem[Cucker and Zhou(2007)Cucker and Zhou]{CuckerZhou2007}
Cucker, F. and Zhou, D.~X. (2007).
\newblock {\em Learning Theory: An Approximation Theory Viewpoint\/}.
\newblock Cambridge University Press, Cambridge.

\bibitem[Denkowski {\em et~al.}(2003)Denkowski, Mig{\'o}rski, and
  Papageorgiou]{DenkowskiEtAl2003}
Denkowski, Z., Mig{\'o}rski, S., and Papageorgiou, N. (2003).
\newblock {\em An introduction to nonlinear analysis: {T}heory\/}.
\newblock Kluwer Academic Publishers, Boston.

\bibitem[Devroye(1982)Devroye]{Devroye1982}
Devroye, L. (1982).
\newblock Any discrimination rule can have an arbitrarily bad probability of
  error for finite sample size.
\newblock {\em IEEE Trans. Pattern Anal. Mach. Intell.}, {\bf 4}, 154--157.

\bibitem[Diestel and Uhl(1977)Diestel and Uhl]{DiestelUhl1977}
Diestel, J. and Uhl, J.~J. (1977).
\newblock {\em Vector Measures\/}.
\newblock American Mathematical Society, Providence, RI.

\bibitem[Dudley(2002)Dudley]{Dudley2002}
Dudley, R.~M. (2002).
\newblock {\em Real Analysis and Probability\/}.
\newblock Cambridge University Press, Cambridge.

\bibitem[Fan {\em et~al.}(2016)Fan, Hu, Wu, and Zhou]{FanHuWuZhou2016}
Fan, J., Hu, T., Wu, Q., and Zhou, D.~X. (2016).
\newblock Consistency analysis of an empirical minimum error entropy algorithm.
\newblock {\em Appl. Comput. Harmonic Anal.}, {\bf 41}, 164--189.

\bibitem[Goodfellow {\em et~al.}(2017)Goodfellow, Bengio, and
  Courville]{Goodfellow-et-al-2016}
Goodfellow, I., Bengio, Y., and Courville, A. (2017).
\newblock {\em Deep Learning\/}.
\newblock Adaptive Computation and Machine Learning. MIT Press, Cambridge, MA.

\bibitem[Hable(2012)Hable]{Hable2012}
Hable, R. (2012).
\newblock Asymptotic normality of support vector machine variants and other
  regularized kernel methods.
\newblock {\em Journal of Multivariate Analysis\/}, {\bf 106}, 92--117.

\bibitem[Hable and Christmann(2011)Hable and Christmann]{HableChristmann2011}
Hable, R. and Christmann, A. (2011).
\newblock Qualitative robustness of support vector machines.
\newblock {\em Journal of Multivariate Analysis\/}, {\bf 102}, 993--1007.

\bibitem[Hu {\em et~al.}(2015)Hu, Fan, Wu, and Zhou]{HuFanWuZhou2015}
Hu, T., Fan, J., Wu, Q., and Zhou, D.~X. (2015).
\newblock Consistency analysis of an empirical minimum error entropy algorithm.
\newblock {\em Anal. Appl.}, {\bf 13}, 437--455.

\bibitem[Huber(1967)Huber]{Huber1967}
Huber, P.~J. (1967).
\newblock The behavior of maximum likelihood estimates under nonstandard
  conditions.
\newblock {\em Proc. 5th Berkeley Symp.}, {\bf 1}, 221--233.

\bibitem[Huber(1981)Huber]{Huber1981}
Huber, P.~J. (1981).
\newblock {\em Robust Statistics\/}.
\newblock John Wiley {\&} Sons, New York.

\bibitem[Mukherjee {\em et~al.}(2006)Mukherjee, Niyogi, Poggio, and
  Rifkin]{MukherjeeNiyogiPoggioRifkin2006}
Mukherjee, S., Niyogi, P., Poggio, T., and Rifkin, R. (2006).
\newblock Learning theory: stability is sufficient for generalization and
  necessary and sufficient for consistency of empirical risk minimization.
\newblock {\em Adv. Comput. Math.}, {\bf 25}, 161--193.

\bibitem[Poggio {\em et~al.}(2004)Poggio, Rifkin, Mukherjee, and
  Niyogi]{PoggioRifkinMukherjeeNiyogi2004}
Poggio, T., Rifkin, R., Mukherjee, S., and Niyogi, P. (2004).
\newblock General conditions for predictivity in learning theory.
\newblock {\em Nature\/}, {\bf 428}, 419--422.

\bibitem[Rio(2013)Rio]{Rio2013}
Rio, E. (2013).
\newblock On {McDiarmid's} concentration inequality.
\newblock {\em Electron. Commun. Probab.}, {\bf 44}, 1--011.

\bibitem[Sch{\"o}lkopf and Smola(2002)Sch{\"o}lkopf and
  Smola]{SchoelkopfSmola2002}
Sch{\"o}lkopf, B. and Smola, A.~J. (2002).
\newblock {\em Learning with Kernels\/}.
\newblock MIT Press, Cambridge, MA.

\bibitem[Shi {\em et~al.}(2011)Shi, Feng, and Zhou]{ShiFengZhou2011}
Shi, L., Feng, Y.~L., and Zhou, D.~X. (2011).
\newblock Concentration estimates for learning with regularizer and data
  dependent hypothesis spaces.
\newblock {\em Applied and Computational Harmonic Analysis\/}, {\bf 31},
  286--302.

\bibitem[Smale and Zhou(2007)Smale and Zhou]{SmaleZhou2007}
Smale, S. and Zhou, D.~X. (2007).
\newblock Learning theory estimates via integral operators and their
  approximations.
\newblock {\em Constr. Approx.}, {\bf 26}, 153--172.

\bibitem[Steinwart and Christmann(2008)Steinwart and Christmann]{SC2008}
Steinwart, I. and Christmann, A. (2008).
\newblock {\em Support Vector Machines\/}.
\newblock Springer, New York.

\bibitem[Steinwart {\em et~al.}(2009)Steinwart, Hush, and
  Scovel]{SteinwartHushScovel2009}
Steinwart, I., Hush, D., and Scovel, C. (2009).
\newblock Optimal rates for regularized least squares regression.
\newblock In {\em Proceedings of the 22nd Annual Conference on Learning
  Theory\/}, pages 79--93.

\bibitem[Steinwart {\em et~al.}(2016)Steinwart, Thomann, and
  Schmid]{SteinwartThomannSchmid2016}
Steinwart, I., Thomann, P., and Schmid, N. (2016).
\newblock Learning with hierarchical {G}aussian kernels, arxiv:1612.00824v1.

\bibitem[Vapnik(1995)Vapnik]{Vapnik1995}
Vapnik, V.~N. (1995).
\newblock {\em The Nature of Statistical Learning Theory\/}.
\newblock Springer, New York.

\bibitem[Vapnik(1998)Vapnik]{Vapnik1998}
Vapnik, V.~N. (1998).
\newblock {\em Statistical Learning Theory\/}.
\newblock John Wiley {\&} Sons, New York.

\bibitem[Wendland(1995)Wendland]{Wendland1995}
Wendland, H. (1995).
\newblock Piecewise polynomial, positive definite and compactly supported
  radial basis functions of minimal degree.
\newblock {\em Adv. Comput. Math.}, {\bf 4}, 389--396.

\bibitem[Wendland(2005)Wendland]{Wendland2005}
Wendland, H. (2005).
\newblock {\em Scattered Data Approximation\/}.
\newblock Cambridge University Press, Cambridge.

\bibitem[Wu(1995)Wu]{Wu1995}
Wu, Z. (1995).
\newblock Compactly supported positive definite radial functions.
\newblock {\em Adv. Comput. Math.}, {\bf 4}, 283--292.

\bibitem[Xiang and Zhou(2009)Xiang and Zhou]{XiangZhou2009}
Xiang, D.~H. and Zhou, D.~X. (2009).
\newblock Classification with gaussians and convex loss.
\newblock {\em Journal of Machine Learning Research\/}, {\bf 10}, 1447--1468.

\bibitem[Ye and Zhou(2007)Ye and Zhou]{YeZhou2007}
Ye, G.~B. and Zhou, D.~X. (2007).
\newblock Fully online classification by regularization.
\newblock {\em Applied Computational Harmonic Analysis\/}, {\bf 23}, 198–214.

\bibitem[Zuo {\em et~al.}(2015)Zuo, Li, and Chen]{ZuoLiChen2015}
Zuo, L., Li, L.~Q., and Chen, C. (2015).
\newblock The graph based semi-supervised algorithm with $\ell^1$-regularizer.
\newblock {\em Neurocomputing\/}, {\bf 149}, 966--974.

\end{thebibliography}
\end{document}